\documentclass[11pt]{amsart}

\usepackage{hyperref}

\usepackage{graphicx, enumerate, url}
\usepackage{amssymb,mathtools}
\usepackage{algorithm}
\usepackage{algpseudocode}
\usepackage{color}
\usepackage{tikz}
\usetikzlibrary{3d,calc}

\numberwithin{equation}{section}
\numberwithin{algorithm}{section}

\theoremstyle{plain}
\newtheorem{theorem}{Theorem}[section]

\newtheorem{lemma}[theorem]{Lemma}
\newtheorem{corollary}[theorem]{Corollary}

\theoremstyle{definition}

\theoremstyle{remark}

\DeclareMathOperator{\tr}{Tr}
\DeclareMathOperator{\diag}{diag}

\DeclareMathOperator{\V}{V}

\DeclareMathOperator{\argmax}{arg\,max}
\DeclareMathOperator{\argmin}{arg\,min}

\newcommand\defeq{\stackrel{\mathclap{\normalfont\mbox{def}}}{=}}
\newcommand{\N}{\mathbb{N}}

\newcommand{\Q}{\mathbb{Q}}
\newcommand{\R}{\mathbb{R}}
\newcommand{\Z}{\mathbb{Z}}
\newcommand{\s}{\mathbf{s}}
\newcommand{\g}{\mathbf{g}}
\newcommand{\x}{\mathbf{x}}

\newcommand{\y}{\mathbf{y}}
\newcommand{\z}{\mathbf{z}}
\newcommand{\U}{\mathbf{U}}
\renewcommand{\V}{\mathbf{V}}

\newcommand{\cA}{\mathcal{A}}
\newcommand{\cB}{\mathcal{B}}
\newcommand{\cC}{\mathcal{C}}
\newcommand{\cD}{\mathcal{D}}
\newcommand{\cF}{\mathcal{F}}
\newcommand{\cG}{\mathcal{G}}

\newcommand{\cJ}{\mathcal{J}}
\newcommand{\cK}{\mathcal{K}}
\newcommand{\cQ}{\mathcal{Q}}

\newcommand{\cU}{\mathcal{U}}
\newcommand{\cV}{\mathcal{V}}
\newcommand{\cW}{\mathcal{W}}
\newcommand{\cX}{\mathcal{X}}
\newcommand{\cY}{\mathcal{Y}}
\newcommand{\ba}{\mathbf{a}}
\newcommand{\bb}{\mathbf{b}}
\newcommand{\bc}{\mathbf{c}}

\newcommand{\bp}{\mathbf{p}}
\newcommand{\bq}{\mathbf{q}}
\newcommand{\br}{\mathbf{r}}
\newcommand{\bt}{\mathbf{t}}
\newcommand{\bu}{\mathbf{u}}
\newcommand{\bv}{\mathbf{v}}
\newcommand{\bw}{\mathbf{w}}
\newcommand{\bx}{\mathbf{x}}
\newcommand{\by}{\mathbf{y}}

\newcommand{\rB}{\mathrm{B}}
\newcommand{\rC}{\mathrm{C}}

\newcommand{\rL}{\mathrm{L}}
\newcommand{\rS}{\mathrm{S}}
\newcommand{\rU}{\mathrm{U}}

\newcommand{\0}{\mathbf{0}}
\newcommand{\1}{\mathbf{1}}
\newcommand{\trans}{^\top}

\begin{document}
\title[TOT, distance between measures and tensor scaling]{Tensor optimal transport, distance between\\ sets of measures and
tensor scaling}
\author[Shmuel Friedland]{Shmuel~Friedland}
\address{Department of Mathematics and Computer Science, University of Illinois at Chicago, Chicago, Illinois, 60607-7045, USA }
\email{friedlan@uic.edu}
\date{July 23, 2021}
\begin{abstract}  
We study the optimal transport problem for $d>2$ discrete  measures. This is a linear programming problem on $d$-tensors.   It gives a way to compute a ``distance'' between two sets of discrete measures. 
We introduce an entropic regularization term, which gives rise to a scaling of tensors.  We give a variation of the celebrated Sinkhorn scaling algorithm.  We show that this algorithm can be viewed as a partial minimization algorithm of a strictly convex function.  Under appropriate conditions the rate of convergence is geometric, and we estimate the rate. Our results are generalizations of known results for the classical case of two discrete measures.
\end{abstract}

\keywords{Tensor optimal transport, entropic relaxation, distance between two sets of discrete  measures,  scaling of nonnegative tensors, partial minimization of strict convex functions, Sinkhorn algorithm}

\subjclass[2010]{15A39, 15A69, 52A41, 62H17, 65D19, 65F35, 65K05, 90C05, 90C25}
\maketitle
\section{Introduction} \label{sec:intro}
For $i\in\{1,2\}$ let $X_i$ be a random variables on $\Omega_i$ which has a finite number of values: $X_i:\Omega_i\to [n_i]$, where $[n]\defeq\{1,\ldots,n\}\subset \mathbb{N}$.  
 Assume that $\bp_i=(p_{1,1},\ldots,p_{n_i,i})$ is the column probability vector that gives the distribution of $X_i$: $\mathbb{P}(X_i=j)=p_{j,i}$.
Then the discrete Kantorovich optimal transport problem (OT) can be states as follows \cite{Kan}.  (See \cite{Vil03,Vil09} for modern account of OT.)  Let $Z$ be a random variable $Z: \Omega_1\times \Omega_2\to [n_1]\times[n_2]$ with contingency matrix (table)  $U\in\R_+^{n_1\times n_2}$ that gives the distribution of $Z$:
$\mathbb{P}(Z=(j_1,j_2))=u_{j_1,j_2}$.  
(Here $\R_+=[0,\infty), \R_{++}=(0,\infty)$.)
Let $P=(\bp_1,\bp_2)$ and  $\rU(P)$ be the convex set of all probability matrices with marginals $\bp_1,\bp_2$:
\begin{eqnarray*}
\rU(P)=\{U=[u_{j_1,j_2}]\in\R_+^{n_1\times n_2}, \sum_{j_2=1}^{n_2} u_{j_1,j_2}=p_{j_1,1}, \quad \sum_{j_1=1}^{n_1} u_{j_1,j_2}=p_{j_2,2}\}.
\end{eqnarray*}
Let $C=[c_{i_1,i_2}]\in \R^{n_1\times n_2}$ be the cost matrix of transporting a unit $j_1\in[n_1]$ to $j_2\in[n_2]$.  Then the optimal transport problem is the linear programming problem (LP): 
\begin{eqnarray*}
\tau(C,P)=\min\{\langle C,U\rangle, U\in \rU(P)\}. 
\end{eqnarray*}
(Here $\langle C,U\rangle=\tr C^\top U$.) 
Thus  $\tau(C,P)$ measures a ``distance'' between probability vectors $\bp_1$ and $\bp_2$, which can be viewed as two histograms, with respect to the cost matrix $C$. 
It turns out that  $\tau(C,P)$ has many recent applications in machine learning \cite{AWR17,ACB17,LG15,MJ15,SL11}, statistics  \cite{BGKL17,FCCR18,PZ16,SR04} and computer vision \cite{BPPH11,RTG00,SGPCBNDG}.

Assume that $n_1=n_2=n$.  Then the complexity of finding $\tau(C,P)$ is $O(n^3 \log n)$, as this problem can be stated in terms of flows \cite{PW09}.  In applications, when $n$ exceeds a few hundreds,
the cost is prohibitive.  One way to improve the computation of $\tau(C,\bp_1,\bp_2)$ is to replace the linear programming with problem of OT with convex optimization by introducing an entropic regularization term as in \cite{Cut13}.  This regularization terms gives an $\varepsilon$-approximation to $\tau(C,P)$, where $\varepsilon>0$ is given.  The regularization term gives almost linear time approximation using a variation of the celebrated Sinkhorn algorithm for matrix diagonal scaling \cite{AWR17}.  

The aim of this paper is to lay down the framework for measuring the ``distance'' between 
two sets of probability measures $\bp_1,\ldots,\bp_l$ and $\bp_{l+1},\ldots,\bp_{d}$ two positive integers $l,m=d-l$.  In the most general case $\bp_i\in\R_+^{n_i}$, where $i\in[d]$.  To simplify the notation we are going to assume in this paper, unless stated otherwise, that $n_1=\cdots=n_d=n$.  Denote by $P$ the matrix $(\bp_1,\ldots,\bp_d)\in\R_{+}^{n\times d}$.
For the case $d=2$ the set $\rU(P)$ was the set of probability matrices $U\in \R_+^{n\times n}$ satisfying the marginal conditions $U \1=\bp_1, U^\top \1=\bp_2$.  Here $\1$ is the vector whose coordinates are all $1$.
For $d\ge 3$ we introduce $d$-mode tensors $\otimes ^d \R^n$.  We denote by $\cU\in  \otimes ^d \R^n$ a tensor whose entries are $u_{i_1,\ldots,i_d}$, i.e., $\cU=[u_{i_1,\ldots,i_d}]$.  Assume that $\cC,\cU \in  \otimes ^d \R^n, \cX\in \otimes^{d-1}\R^{n}$.   Denote by $\langle \cC,\cU\rangle$ the Hilbert-Schmidt inner product $\sum_{i_1,\ldots,i_d\in[n]} c_{i_1,\ldots,i_d}u_{i_1,\ldots,i_d}$.  For $k\in [d]$ denote by  $\y=\cU\times_k \cX\in \R^n$ the contraction on all but the index $k$:
$y_{i_k}=\sum_{i_j\in[n], j\ne k}u_{i_1,\ldots,i_d} x_{i_1,\ldots,i_{k-1},i_{k+1},\ldots,i_d}$.
Let $\cJ_{d-1}\in\otimes^{d-1}\R^n$ be the tensor whose all coordinates are $1$.
Define
\begin{eqnarray*}
\rU(P)=\{\cU\in \otimes^d\R_+^n, \cU\times_k \cJ_{d-1}=\bp_k, k\in[d]\}, \quad P=(\bp_1,\ldots,\bp_d).
\end{eqnarray*}
Then the tensor optimal transport problem (TOT) is
\begin{eqnarray}\label{TOT}
\tau(\cC,P)\defeq \min\{\langle \cC, \cU\rangle, \cU\in \rU(P)\}.
\end{eqnarray}
TOT problem is a LP problem with $n^d$ nonnegative variables and $d(n-1)+1$ constraints. 

The TOT was considered in \cite{Pi68,Po94} in the context of multidimensional assignment problem, where the entires of the tensor $\cU$ are either $0$ or $1$.
There is a vast literature on continuous multidimensional optimal transport problem.
See for example \cite{FV18} and the references therein. 
The TOT problem can be viewed as a discretization the  continuous multidimensional optimal transport problem.

The main result of this paper is to give an approximate solution $\cB\in \rU(P)$ to TOT that satisfies the inequality
\begin{eqnarray*}
\langle C,\cB\rangle\le  \tau(\cC,P)+\delta,
\end{eqnarray*}
for a given $\delta>0$.  The number of operations to find $\cB$ is $O(\frac{\omega^3d^4 n^{d+1}\log n}{\delta^3})$.  Here $\omega$ is the difference between the maximum and the minimum of the entries of $\cC$.  Our results are generalizations of the corresponding results in \cite{Cut13,AWR17} for the matrix case $d=2$.

We expect that our results will find various applications in the problems that need to compare two sets of discrete measures.
\subsection{Summary of the paper}\label{subsec:sump}   In Subsection \ref{subsec:Was} 
we introduce the notion of the Wasserstein distance between two ordered sets of probability vectors, each containing $d/2$ discrete measure, with $d\ge 4$.  As in the case of $d=2$ it can be stated in terms of TOT on $d$-mode tensors. 
We discuss briefly the problem of giving  the Wasserstein distance between two unordered sets of probability vectors in terms of TOT.  In Subsection \ref{subsec:entrel} we discuss the entropic relaxation of TOT:
\begin{eqnarray*}
\min\{\langle \cC,\cU\rangle-(1/\lambda) H(\cU), U\in\rU(P)\}.
\end{eqnarray*}
Here $\lambda>0$ and $H(U)$ is the entropy of a probability tensor.  We prove a generalization of the result in \cite{Cut13} for matrices.  Namely minimum of the entropic relaxation is achieved at the unique tensor $\cU$ which is a scaling of the positive tensor $\exp(-\lambda\cA)$, where $\exp(\cB)=[e^{b_{i_1,\ldots,i_d}}]$ is the tensor entry-wise exponent.  It is well known that there is a unique scaling of $\cU\in\rU(P)$ of given positive tensor $\cA\in\otimes^d\R^n$, provided that $\bp_1,\ldots,\bp_d$ are positive probability vectors \cite{Bap82,Rag84}.  This result can be extended to nonnegative tensors under the condition that there is $\cU\in\rU(P)$ with the same zero pattern as $\cA$ \cite{BR89, FL89, NR99,Fri11}.

The main results of this paper are given in Section \ref{sec:tenscal}.  In this section we 
give a generalized version of the celebrated Sinkhorn algorithm to tensor scaling, abbreviated here as SA.  Using our results on partial minimization of strictly convex function developed in Appendix \ref{appendix:pm} we show that SA converges geometrically to the unique scaling of $\cA$.
We analyze the convergence of this algorithm as in \cite{AWR17}.  Namely we can approximate the value of TOT
within $\delta$ error and find an approximate solution as explained above.
\section{Entropic relaxation}\label{sec:entrelax}
\subsection{The Wasserstein distance between sets of probabity vectors}\label{subsec:Was}
Recall that a set $\mathbf{X}$ is a metric space with respect to $\delta: \mathbf{X}\times \mathbf{X}\to [0,\infty)$ if the following three conditions hold
\begin{itemize}
\item $\delta(x,x)=0$, (zero selfdistance);
\item $\delta(x,y)=\delta(y,x)>0$ for  $x\ne y$ (symmetricity and positive distance);
\item $\delta(x,y)\le \delta(x,z) +\delta(z,y)$ (triangle inequality).
\end{itemize}
Assume that $\mathbf{X}$ is a finite set of points. We can identify $\mathbf{X}$ with $[N]$, where the cardinality of $\mathbf{X}$ is $N$.  The distance $\delta$ induces the distance matrix $\Delta=[\delta(i,j)]\in\R_+^{N\times N}$.  
Note that $\Delta$ is a nonnegative symmetric matrix with zero diagonal and positive off-diagonal entries.  The matrix $\Delta$ is called  uniform if $\delta(i,j)=t>0$ for all $i\ne j\in[N]$, and  almost uniform if $\delta(i,j)\in [t/2, t]$ for all $i\ne j\in[N]$.  We call $\Delta$ a normalized uniform or almost normalized uniform if $t=1$.

Let $\mathbf{X}$ be the simplex of probability vectors $\Pi^n\subset \R^n$.  Denote by  $\Pi^n_o$ the open set of positive probability measures. Recall that OT gives rise to
the Wasserstein, also known as Kantorovich-Rubenstein, metric on $\Pi^n$ \cite{Vil09}. 
In this subsection we discuss a generalization of the Wasserstein metric to $(\Pi^n)^l$ for $l>1$.

We first start with a trivial observation that $\rU(P)$ is nonempty as it contains the tensor product $\otimes_{i=1}^d \bp_i$.   A tensor $\cU\in \rU(P)$ is called a contingency tensor.  (For simplicity of exposition we assume that $d\ge 1$, where $\rU(\bp_1)=\{\bp_1\}$.)
Let $\cC\in\otimes^d\R_+^n$. 
Note that in the definition of  $\tau(\cC,P)$ one does not see the partition of the probability vectors to two sets $\{\bp_1,\ldots,\bp_l\}$ and $\{\bp_{l+1},\ldots,\bp_d\}$.
Assume that $d=2l$.  Define $P_1=(\bp_1,\ldots,\bp_{d/2}), P_2=(\bp_{d/2+1},\ldots,\bp_d)\in\R^{n\times (d/2)}$.  Thus $P=(P_1,P_2)$.
We want to give simple conditions on $\cC$, such that $\tau(\cC,P)$ is a distance between $P_1$ and $P_2$, denoted as dist$(P_1, P_2)$.  For $d=2$ the marix $\cC=[c_{i_1,i_2}]$ is a distance matrix $\Delta$ \cite[Sec. 6.1]{Vil09}.  

We now generalize the notion of a distance matrix $\Delta$ to a distance tensor $\cC$ as follows.  Let $D(\cC)\in \R_+^{n^{d/2}\times n^{d/2}}$ be the following matrix induced by $\cC$.  Index the rows and the columns of $D(\cC)$ by the sets $[n]^{d/2}$. 
Thus the entries of $D(\cC)$ are $d_{(i_1,\ldots,i_{d/2}),(i_{d/2+1},\ldots,i_d)}$.  Then $d_{(i_1,\ldots,i_{d/2}),(i_{d/2+1},\ldots,i_d)}\defeq c_{i_1,\ldots,i_d}$.
\begin{lemma}\label{disttensor}
Let $d/2\in\N$ and assume that $\cC\in\otimes^d\R_+^n$.  Let $D(\cC)\in\R_+^{n^{d/2}\times n^{d/2}}$ be the matrix defined as above.  If $D(\cC)$ is a distance matrix then $\tau(\cC,\cdot):(\Pi^n)^{d/2}\times (\Pi^n)^{d/2}\to [0,\infty)$ is a distance on the space $(\Pi^n)^{d/2}$.
\end{lemma}
\begin{proof} Clearly $\tau(\cC,(P_1,P_2))\ge 0$.  Let $P_3=(\bp_{d+1},\ldots,\bp_{3d/2})\in (\Pi^n)^{d/2}$.
We claim that 
$\tau(\cC,(P_1,P_1))=0$.  Let $\cU=[u_{i_1,\ldots,i_d}]\in \otimes^d\R_+^n$ be of the following form: 
\begin{eqnarray*}
u_{i_1,\ldots,i_d}=
\begin{cases} 
0 &\textrm{ if }(i_1,\ldots,i_{d/2})\ne (i_{d/2+1},\ldots,i_{d}),\\
\prod_{j=1}^{d/2} p_{i_j,j}&\textrm{ if }(i_1,\ldots,i_{d/2})= (i_{d/2+1},\ldots,i_{d}).
\end{cases}
\end{eqnarray*}
Then $\cU\in\rU((P_1,P_1))$.  As the diagonal enties of $\Delta(\cC)$ are zero we deduce that $\langle \cC,\cU\rangle=0$.  Hence $\tau(\cC, (P_1,P_1))=0$.
Assume now that $P_1\ne P_2$.  We claim that
$\tau(\cC,(P_1,P_2))>0$.  Suppose that $\cU\in\rU((P_1,P_2))$ is optimal: 
$\tau(\cC,(P_1,P_2))=\langle \cC,\cU\rangle$.  Assume to the contrary that $\langle C,\cU\rangle=0$.  Then $u_{i_1,\ldots,i_d}=0$ if $(i_1,\ldots,i_{d/2})\ne  (i_{d/2+1},\ldots,i_{d})$.  As  $\cU\in\rU((P_1,P_2))$ we deduce that
\begin{eqnarray*}
p_{i_k,k}=\sum_{i_j\in[n],j\ne k} u_{i_1,\ldots,i_d}=\sum_{i_j\in[n], j\in[d/2]\setminus\{k\}} u_{i_1,\ldots,i_{d/2},i_1,\ldots,i_{d/2}} \textrm{ for } k\in[d/2].
\end{eqnarray*}
Assume now that $k=q+d/2$ for $q\in[d/2]$.  The above arguments show that $p_{i_k,k}=p_{i_k,q}$.  Hence $\bp_k=\bp_q$ for all $q\in[d/2]$.  This contradicts the assumption that $P_1\ne P_2$.

It is left to show the triangle inequality.
We show how to reduce our case to the case $d=2$.   Assume that $\cU\in \rU((P_1,P_2))$.  Then $\cU$ is the distribution of a random variable $Z:\Omega^d\to [n]^d$.  Hence $Z$ is the joint distribution $(X_1,X_2)$, where  $X_i:\Omega^{d/2}\to [n]^{d/2}$.  The distribution of $X_i$ is given by $\cX_i=[x_{j_1,\ldots,j_{d/2},i}]$.
Thus 
\begin{eqnarray*}
x_{j_1,\ldots,j_{d/2},1}=\sum_{j_{d/2+1},\ldots,j_d\in[n]} u_{j_1,\ldots,j_d}, \quad 
x_{j_{d/2+1},\ldots,j_d,2}=\sum_{j_{1},\ldots,j_{d/2}\in[n]} u_{j_1,\ldots,j_d}.
\end{eqnarray*}
Note that $\cX_1\in\rU(P_1), \cX_2\in\rU(P_2)$.
We next recall the Gluing Lemma \cite{Vil09}.  Assume that $P_3=(\bp_{d+1},\ldots,\bp_{3d/2})\in (\Pi^n)^{d/2}$.  Let 
$\cV\in \rU((P_2,P_3))$.  Assume that $W$ be the joint distribution of $(X_2,X_3)$ given by $\cV$.  Then there exists a joint distribution of $(X'_1,X'_2,X'_3)$ such the joint distribution of $(X_1',X'_2)$ and $(X_2',X_3')$ are the distribution of $Z$ and $W$ respectively.  To show that it is enough to assume that $d=2$.  Then the joint distribution of $(X'_1,X'_2,X'_3)$ is given by $p_{i_2,2}^{-1} u_{i_1,i_2}v_{i_2,i_3}$. 
(We assume here that $0/0=0$.) 
We now revert to the assumption that $d/2\in\N$.
Observe that the joint distribution of $(X_1',X_3')$ is given by $\cQ\in\rU((P_1,P_3))$.  Hence $\tau(\cC,(P_1,P_3))\le \langle \cC,\cQ\rangle$.
Now choose the optimal distribution $\cU$ and $\cV$ for $\tau(\cC,(P_1,P_2))$ and $\tau(\cC,(P_2,P_3))$ respectively.
Let $(X_1',X_2',X_3')$ be the join distribution as above. given by $\cW=[w_{j_1,\ldots,j_{3d/2}}]$.  Consider the inequality 
\begin{eqnarray*}
&&c_{j_1,\ldots,j_{d/2},j_{d+1},\ldots,j_{3d/2} } w_{j_1,\ldots,j_{3d/2}}\le \\
&&(c_{j_1,\ldots,j_{d}}+c_{j_{d/2+1},\ldots,j_{3d/2} }) w_{j_1,\ldots,j_{3d/2}}
\end{eqnarray*}
Sum on all indices.  Then we get the triangle inequality 
\begin{eqnarray*}
\tau(\cC,(P_1,P_3))\le\langle \cC,\cQ\rangle\le \langle \cC,\cU\rangle +\langle \cC,\cW\rangle=
\tau(\cC,(P_1,P_2))+\tau(\cC,(P_2,P_3)).
\end{eqnarray*}
\end{proof}
 
Assume that  $D(\cC)$ is a distance matrix.  In many cases it would be convenient to have that the quantity $ \tau(\cC,(P_1,P_2))$ to be the distance between the multisets $\{\bp_1,\ldots,\bp_{d/2}\}$ and $\{\bp_{d/2+1},\ldots,\bp_{d}\}$.  (We do not assume that $\bp_1,\ldots,\bp_{d/2}$ are pairwise distinct.) 
This seems to be not an easy task for the following reason. 

Let us find $\cC\in\otimes^d\R$ for which $\tau(\cC,p_1,\ldots,p_d)$ is invariant under  the permutations on the sets $\{\bp_1,\ldots,\bp_{d/2}\}$ and $\{\bp_{d/2+1},\ldots,\bp_{d}\}$, and symmetric in $\{\bp_1,\ldots,\bp_{d/2}\}$ and $\{\bp_{d/2+1},\ldots,\bp_{d}\}$.
To find such cost tensors we define the class of bisymmetric tensors $\rB^d\R^n\subset\otimes^d\R^n$.  Let $\Sigma_k:[k]\to[k]$ be the group of bijections (permutations).  Then $\cC\in \rB^d\R^n$ if the following conditions hold:
\begin{eqnarray*}
&&c_{i_1,\ldots,i_d}=c_{i_{\alpha(1)},\ldots,i_{\alpha(d/2)}, i_{ d/2+\beta(1)},\ldots i_{ d/2+\beta(d/2)}}\; \forall \alpha,\beta\in \Sigma_{d/2},\\
&&c_{i_1,\ldots,i_d}=c_{i_{d/2+1},\ldots,i_d,i_1,\ldots,i_{d/2}}.
\end{eqnarray*}
For $\alpha\in \Sigma_{d/2}$ denote $\alpha(P_1)=(\bp_{\alpha(1)},\ldots,\bp_{\alpha(d/2)})$.
It is straightforward to see that for $\cC\in\rB^d\R^n$ we have the following equalities
\begin{eqnarray*}
\tau(\cC,(P_1,P_2))=\tau(\cC,(P_2, P_1)),\,\tau(\cC,(P_1,P_2))=\tau(\cC,(\alpha(P_1),\beta(P_2)),
\end{eqnarray*}
for all $\alpha,\beta\in \Sigma_{d/2}$.

Assume that $\cC\in \rB^d\R^n_+$.  Then $D(\cC)$ is called a bisymmetric distance matrix if 
\begin{eqnarray*}
&&c_{i_1,\ldots,i_d}=
\begin{cases}
0&\textrm{ if } \{i_1,\ldots,i_{d/2}\}=\{i_{d/2+1},\ldots,i_d\},\\
>0&\textrm{ if } \{i_1,\ldots,i_{d/2}\}\ne\{i_{d/2+1},\ldots,i_d\},
\end{cases}
\\
&&c_{i_1,\ldots,i_d}\le c_{i_1,\ldots,i_{d/2},i_{d+1},\ldots,i_{3d/2}}+c_{i_{d+1},\ldots,i_{3d/2},i_{d/2+1},\ldots,i_{d}}.
\end{eqnarray*}
Here by $\{i_1,\ldots,i_{d/2}\}$ we mean a multiset of indices in $[n]$.
It is easy to show that that if $D(\cC)$ is a bisymmetric distance matrix then $\tau(\cC,(P_1,P_2))$ is a semimetric.  That is the second condition of metric, named positive distance, may not hold.  One can trivially amend this situation if we don't insist that the metric on $(\Pi^n)^{d/2}$ is continuous as in \cite{Cut13}.
Assume that  $\cC\in \rB^d\R^n_+$ has positive entries and satisfies the last of the above conditions.   Then $\tau(\cC,(P_1,P_2))$ satisfies the last two condition of the metric, but not  the condition $\delta(x,x)=0$.  Let $1_{P_1,P_2}: (\Pi^n)^{d/2}\times (\Pi^n)^{d/2}\to \{0,1\}$ be the function such that $1_{P_1,P_2}=0$ if and only if $\{\bp_1,\ldots,\bp_{d/2}\}=\{\bp_{d/2+1},\ldots,\bp_{d}\}$.   (That is $P_1=\beta(P_2)$ for some $\beta\in \Sigma_{d/2}$.) Then $1_{P_1,P_2}\tau(\cC,(P_1,P_2))$ is a metric discontinuous on the diagonal.

It is simple to find $\cC$ such that $ \tau(\cC,(P_1,P_2))$ is invariant distance under the identical  permutations on the sets $\{\bp_1,\ldots,\bp_{d/2}\}$ and $\{\bp_{d/2+1},\ldots,\bp_{d}\}$.  Denote by $\rB^d_w\R^n$ the set of weak bipartite symmetric tensor, where we assume in the definition of bipartite tensors that $\alpha=\beta$.  Assume that $\cC\in\rB^d_w\R_+^n$ and the matrix $D(\cC)$ is a distance matrix.  Then $ \tau(\cC,(P_1,P_2))=\tau(\cC,(\alpha(P_1),\alpha(P_2)))$ for all $\alpha\in \Sigma_{d/2}$.

It is easy to get a distance on $(\Pi^n)^{d/2}$ which is invariant under permutation on the sets $\{\bp_1,\ldots,\bp_{d/2}\}$ and $\{\bp_{d/2+1},\ldots,\bp_{d}\}$, which is induced by TOT function.
Take $\cC\in \rB^d_w\R^n$ and assume that $D(\cC)$ is a distance matrix.  Then
\begin{eqnarray*}
\textrm{dist}(\{\bp_1,\ldots,\bp_{d/2}\},\{\bp_{d/2+1},\ldots,\bp_d\})\defeq
\min\{\tau(\cC,(\alpha(P_1),\alpha(P_2))), \alpha\in\Sigma_{d/2}\}
\end{eqnarray*}
is a distance between the sets $\{\bp_1,\ldots,\bp_{d/2}\}$ and $\{\bp_{d/2+1},\ldots,\bp_{d}\}$.
\subsection{Entropic relaxation}\label{subsec:entrel}
For $\cU\in\rU(P)$ we define the entropy 
 the  contingency tensor as follows:
\begin{eqnarray*}
H(\cU)=-\sum_{i_1,\ldots,i_d\in[n]} u_{i_1,\ldots,i_d}\log u_{i_1,\ldots,i_d}.
\end{eqnarray*}
Assume that $\cU\in\rU(P)$  is the common distribution of   $(X_1,\ldots,X_d)$.  Then $H(X_1,\ldots,X_d)=H(\cU)$.  Recall that $H(X_1,\ldots,X_d)\le \sum_{j=1}^d H(X_j)$ \cite[Thorem 2.6.6]{CT91}.  Equality holds if and only if $X_1,\ldots,X_d$ are independent random variables, Hence $H(\cU)\le \sum_{j=1}^d H(\bp_j)$. Equality holds if and only if $\cU=\otimes_{j=1}^d\bp_j$.   Hence 
\begin{eqnarray*}
0\le H(\cU)\le \sum_{j=1}^d H(\bp_j)\le d \log n.
\end{eqnarray*}
As in \cite{Cut13} for $\lambda>0$ we introduce the entropic relaxation of 
$\tau(\cC,P)$:
\begin{eqnarray}\label{entrelprob}\quad
\min\{\langle \cC,\cU\rangle - \frac{1}{\lambda} H(\cU), U\in\rU(P)\}=\langle \cC,\cU_\lambda\rangle - \frac{1}{\lambda} H(\cU_\lambda).
\end{eqnarray}
Denote by $\rU_{n,d}\subset\otimes^d\R_+^n$ the convex set of all probability tensors, i.e., all nonnegative $\cU$ such that $\langle \cJ_d,\cU\rangle=1$.
Let $f_\lambda:\rU_{n,d}\to \R$ be the function $f_\lambda(\cU)=\langle \cC,\cU\rangle -\frac{1}{\lambda}H(\cU)$.  As $-H(\cU)$ is a strict convex function on $\rU_{n,d}$ it follows that $f_\lambda$ is a strict convex function on $\rU_{n,d}$.
Hence $\cU_\lambda$ is the unique solution to the entropic relaxation problem.
Note that 
\begin{eqnarray*}
f_{\lambda}(\cU)\le \langle \cC, \cU\rangle\le f_{\lambda}(\cU)+\frac{d\log n}{\lambda},\quad \forall \cU\in\rU(\bp_1,\ldots,\bp_d).
\end{eqnarray*}
Hence
\begin{eqnarray*}
f_\lambda(\cU_\lambda)\le \tau(\cC,P)\le f_{\lambda}(\cU_\lambda)+\frac{d\log n}{\lambda}.
\end{eqnarray*}
We now state a characterization of $\cU_\lambda$ in terms of the notion of tensor scaling.  We say that two nonnegative tensors $\cU,\cV\in\otimes^d\R^n_+$ have the same zero pattern if $u_{i_1,\ldots,i_d}=0$ if and only if $v_{i_1,\ldots,i_d}=0$.  We say that $\cV\in \otimes^d\R^n_+$ is scalable to $\cU$  if there exist vectors  $\x_{j}=(x_{1,j},\ldots,x_{n,j})\in\R^n, j\in[d]$ such that $[\exp(x_{i_1,1}+\cdots +x_{i_d,d})v_{i_1,\ldots,i_d}]=[u_{i_1,\ldots,i_d}]$.  The following result is known \cite{BR89, FL89, NR99,Fri11}:
\begin{theorem}\label{diagscalt}  Let $\cU\in\rU(P)$ where $P\in\Pi^n_o$.  Assume that $\cV\in\R^n_+$ has the same zero pattern as $\cU$.  Then there exists a unique scaling $\cW$ of $\cV$ such that $\cW\in\rU(P)$.
\end{theorem}
Observe that if $\bp_1,\ldots,\bp_d\in\Pi^n_o$ then $\cU=\otimes_{j=1}^d \bp_j$ is a positive tensor in $\rU(P)$. Hence every positive tensor $\cV\in \otimes^d\R_+^n$ can be scaled to a unique $\cW\in \rU(P)$.  
Note that the problem: For a given $\cV\in\R^n_+$ does there exists $\cU\in\rU(P)$ with the same zero pattern as $\cV$? is a LP problem \cite{Fri11}.

The following result is a generalization of the result in \cite{Cut13} for matrices:
\begin{theorem}\label{Ulambform} Assume that $\cC\in\otimes^d \R^n$, $\lambda>0$ and $\bp_1,\ldots,\bp_d\in\Pi^n_o$.
Consider the the entropic relaxation problem \eqref{entrelprob}.   Then $\cU_\lambda$ is the unique scaling of $\exp(-\lambda \cC)$.
\end{theorem}
\begin{proof}  As $\exp(-\lambda \cC)$ is a positive tensor Theorem \ref{diagscalt} yields that there exists a unique $\cV\in\rU(P)$ which is a scaling of  $\exp(-\lambda \cC)$: 
\begin{eqnarray*}
&&v_{i_1,\ldots,i_d}=\exp(x_{i_1,1}+\cdots+x_{i_d,d})\exp(-\lambda c_{i_1,\ldots,i_d}),\\ &&\x_j=(x_{i_1,j}\ldots,x_{i_d,j})\in\R^n, j\in[d].
\end{eqnarray*}
Denote $\log \cV\defeq[\log v_{i_1,\ldots,i_d}]\in\otimes^d\R^n$.  Observe that
\begin{eqnarray*}
\log \cV=-\lambda \cC+\cY, \quad y_{i_1,\ldots,i_d}=\sum_{j=1}^d x_{i_j,j}.
\end{eqnarray*}
As $f_\lambda$ is strictly convex on $\rU(P)$ it is enough to show that $\cV$ is a critical point of $f_\lambda|\rU(P)$.  Let $\cW\in \otimes^d\R^n$ be a nonzero solution to the system of homogeneous linear equations
\begin{eqnarray*}
\sum_{i_j\in[n], j\ne k}w_{i_1,\ldots,i_d}=0, \textrm{ for all } i_k\in[n], k\in[d].
\end{eqnarray*}
As $\langle \cJ_d,\cW\rangle =0$, it follows that this system is a system of $d(n-1)+1$ linear equations in $n^d$ unknowns.
Since $\cV$ is a positive tensor there exists $a>0$ such that $\cV+t\cW\in \rU(P)$ for $t\in (-a,a)$.  Let us compute the derivative of $f_\lambda$ in the  direction of $\cW$ at $\cV$:
\begin{eqnarray*}
&&\frac{d\;}{dt}f_\lambda(\cV+t\cW)|_0=\langle \cC,\cW\rangle +\frac{1}{\lambda}\langle \cJ_d,\cW\rangle +\frac{1}{\lambda}\langle \cW, \log \cV\rangle=\\
&&\langle \cC,\cW\rangle -\langle \cW,\cC\rangle +\frac{1}{\lambda}\langle \cW,\cY\rangle=\frac{1}{\lambda}\sum_{k=1}^d\sum_{i_k\in[n]}x_{i_k,k}\sum_{i_j\in[n],j\ne k}w_{i_1,\ldots,i_d}=0.
\end{eqnarray*}
Thus $\cV$ is a critical point of $f_\lambda|\rU(\bp_1,\ldots,\bp_d)$.  Therefore $\cU_\lambda=\cV$.
\end{proof}
\section{Tensor scaling}\label{sec:tenscal}
For $s\in[1,\infty]$ and $\x=(x_1,\ldots,x_n)\in\R^m$ denote $\|\x\|_s=\bigl(\sum_{i=1}^n |x_i|^s\bigr)^{1/s}$. For convenience we let $\|\x\|\defeq \|\x\|_2$.
For $\x_1,\ldots,\x_n\in\R^n$ we denote $(\x_1,\ldots,\x_d)\in (\R^n)^d$ by the matrix $X\in\R^{n\times d}$.
In this section we assume that $\bp_j\in \R^n_{++},\|\bp_j\|_1=h$ for $j\in[d]$.
Let $P=(\bp_1,\ldots,\bp_d)$.
We denote 
\begin{eqnarray*}
\rU(P)=\{\cU\in \otimes^d\R^n_+, \cU\times_j\cJ_{d-1}=\bp_j,j\in[d]\}.
\end{eqnarray*}
We assume that $\cA\in\otimes^d\R^n_+$, and $\cA\times_j \cJ_{d-1}$ are positive vectors for $j\in[d]$.  Let $\eta$ and $\theta$ be the minimum and the maximum value of positive entries of $\cA$.
For complexity analysis of algorithms we assume that either $\bp_1,\ldots,\bp_d\in\Q_{++}^n$ or $\bp_1,\ldots,\bp_d\in\Z_{++}^n$, and $\cA\in \Q_+^{n\times n}$.

\subsection{Minimization of convex functions and tensor scaling}\label{subsec:mints}
Let $g_\cA$ and $\tilde g_\cA$be the following two functions on $(\R^n)^d$:
\begin{eqnarray}\label{defgC}
&&g_\cA(Y)=\sum_{i_1,\ldots,i_d\in[n]}a_{i_1,\ldots,i_d}\exp(\sum_{j\in[d]}y_{i_j,j})-\sum_{j\in[d]} \sum_{i_j\in[n]}p_{i_j,j}y_{i_j,j},\\
&& \tilde g_\cA(Y)=\sum_{i_1,\ldots,i_d\in[n]}a_{i_1,\ldots,i_d}\exp(\sum_{j\in[d]}y_{i_j,j}), \quad Y=(\y_1,\ldots,\y_d).
\label{deftildegC}
\end{eqnarray}
Clearly these  functions are convex,  We will show that for $d\ge 2$ these two functions are not strictly convex.
As $g_\cA-\tilde g_\cA$ is a linear function it follows that  $g_\cA$ is strictly convex on
an interval of positive length
\begin{eqnarray*}
[Y,Z]=\{t Y+(1-t)Z,t\in[0,1]\}, \quad Z=(\z_1,\ldots,\z_d),
\end{eqnarray*}
if and only if $\tilde g_\cA$ is stricly convex on this interval.

Recall that  $\mathbf{A}$ is an affine subspace of $ (\R^n)^d$ if for  every $r$ points $\y_1,\ldots,\y_r\in\mathbf{A}$ the point $\sum_{i=1}^r a_i\y_i$ is in $\mathbf{A}$ if $\sum_{i=1}^r a_i=1$.
It is well known that $\mathbf{A}$ is an affine subspace if and only if $\mathbf{A}=Y+ \mathbf{U}$ for some subspace $\mathbf{U}\subset (\R^n)^d$.
Consider the following $d-1$ dimensional subspace  $(\R^n)^d$:
\begin{eqnarray*}
\mathbf{U}_{n,d}=\{(\y_1,\ldots,\y_d)\in(\R^n)^d, \y_j=t_j\1, t_j\in\R, \sum_{j=1}^d t_j=0\}
\end{eqnarray*} 
Let $\mathbf{A}=Y+ \mathbf{U}_{n,d}$.
Then the function $\tilde g_\cA$ is constant on $\mathbf{A}$.  Hence $\tilde g_\cA$ is not strictly convex.

The following lemma is well known and we bring its proof for completeness:
\begin{lemma}\label{rminlem}  Let $\cA\in \otimes^d\R^n_+\setminus\{0\}$.  Then the following statements are equivalent:
\begin{enumerate}[\upshape (a)]
\item $\cA$ is scalable  to $\cU\in\rU(P)$. 
\item The function $g_\cA$ achieves minimum on $(\R^n)^d$.
\item The function $g_\cA$ achieves minimum on the affine space
\begin{eqnarray*}
\mathbf{B}(P,\bb)=\{(\y_1,\ldots,\y_d)\in((\R^n)^d, \langle \bp_j,\y_j\rangle=b_j, j\in[d]\}, \quad \bb=(b_1,\ldots,b_d).
\end{eqnarray*}
\item Let $\mathbf{B}(P)\subset \mathbf{B}(P,\0)$ be the maximal subspace on which  
$g_\cA$ is constant.  Let $\mathbf{C}$ the orthogonal complement of $\mathbf{B}(P)$ in $\mathbf{B}(P,\0)$.  The function $g_\cA$ achieves minimum on the subspace $\mathbf{C}$.
\end{enumerate}
\end{lemma}
\begin{proof} We first recall a well known fact:  Let $\mathbf{A}$ an affine subspace of $(\R^n)^d$.   Then $g_\cA|\mathbf{A}$ is a convex function.  Hence a point $X^\star\in \mathbf{A}$ is critical if and only if it is a minimum point.

\noindent
(a)$\iff$(b).   Observe that  $X^\star$ is a critical point of $g_\cA$ if and only if $\cU=[\exp(\sum_{j\in[d]} x_{i_j,j}^\star) a_{i_1,\ldots,i_d}]\in \rU(P)$.

\noindent
(b)$\Rightarrow$(c)  Assume that $X^\star$ is a minimium point of $g_\cA$.  Let $\langle \bp_j,\x_j^\star\rangle=a_j$ for $j\in[d]$.  Then $X^\star$ is a minimum point on 
$\mathbf{B}(P,\ba)$.  Let $\bt=(t_1,\ldots,t_d)\in\R^d$. Consider the  translation $\phi_{\bt}:(\R^n)^d \to (\R^n)^d$ given by $\phi_{\bt}(\y_1,\ldots,\y_d)=(\y_1+t_1 \1,\ldots,\y_d+t_d \1)$. Then we have the following two equalities:
\begin{eqnarray*}
g_{\cA}(\phi_{\bt}(Y))= e^Tg_\cA(Y)+\sum_{j\in[d]}((1-e^T)\langle \bp_j,\y_j\rangle+t_j h),
\end{eqnarray*}
 for $T=t_1+\cdots +t_d$.  Furthermore $\phi_{\bt}(\mathbf{B}(P,\ba))= \mathbf{B}(P,\ba+h\bt)$.
Therefore $\phi_{(1/h)(\bb-\ba)(b_1-a_1)/h}(X^\star)$ is a minimal point of 
$g_\cA|\mathbf{B}(P,\bb)$.

\noindent
(c)$\iff$(d)   Define
\begin{eqnarray}\label{defsubsC}
\begin{aligned}
\mathbf{B}&=\{(\y_1,\ldots,\y_d)\in  \mathbf{B}(P,\0), \sum_{j=1}^d y_{i_j}=0 \textrm{ if } a_{i_1,\ldots,i_d}>0\},\\
\mathbf{C}&=\mathbf{B}^\perp\subset \mathbf{B}(P,\0).
\end{aligned}
\end{eqnarray}
Since $e^t$ is strictly convex it follows the restriction of $g_\cA$ to the interval $[Y,Z]$ of positive length
 is not strictly convex if and only if $Z-Y)\in\mathbf{B}$ \cite{Fri11}.  Hence $\mathbf{B}=\mathbf{B}(P)$ is the maximal subspace in  $\mathbf{B}(P,\0)$ on which $g_\cA$ is constant.  Therefore, for each $Y\in \mathbf{B}(P,\0)$,  $Z\in \mathbf{B}\setminus\{\0\}$ the function $g_{\cA}$ is constant on the line $Y+tZ, t\in\R$.
 Hence the assumption (c) yields that $g_\cA$ attains minimum on $\mathbf{C}$.  Vice versa, assume that $g_\cA|\mathbf{C}$ has minimum. As  
$\mathbf{B}(P,\0)=\mathbf{B}\oplus\mathbf{C}$ we deduce that $g_\cA$ has minimum on $\mathbf{B}(P,\0)$.
 
 \noindent
(c)$\Rightarrow$(a)  Assume that $Z^\star\in\mathbf{B}(P,\0)$ is a minimum point of $g_\cA$.  Let $\cV=[\exp(\sum_{j\in[d]}z_{i_j,j}^\star)a_{i_1,\ldots,i_d}]$. Use Lagrange mulitpliers to deduce that $\cV\times_k \cJ_{d-1}=t_k\bp_k$.  Hence $\langle \cV,\cJ_d\rangle=t_k$.  As $\cA\gneq 0$ we deduce that $t_1=\ldots =t_d>0$.  Thus $t^{-1}\cV\in \rU(P)$.
\end{proof}
In this lemma we estimate the eigenvalues of the Hessian of $g_{\cA}|\mathbf{C}$:
\begin{lemma}\label{eigHesest}
Let $\cA\in\otimes^d\R^n_+\setminus{0}$.  Denote by $\alpha$ and $\beta$ the minimum and the maximum of nonzero entries of $\cA$.  Let $\mathbf{C}$  be defined by \eqref{defsubsC}.  Denote  $\hat g_\cA:\defeq g_\cA|\mathbf{C}$.  Then $\hat g_\cA$ is strictly convex on $\mathbf{C}$.  Assume that $Y\in \mathbf{C}$.    Then the eigenvalues of the Hessian of $\hat g_\cA(Y)$, with respect to an orthonormal basis in $\mathbf{C}$, lie in the interval $[\alpha(Y),\beta(Y)]$ given by
\begin{eqnarray*}
\log\alpha(Y)=\min\{\log a_{i_1,\ldots,i_d} +\sum_{j=1}^d y_{i_j}, i_1,\ldots,i_d\in[n], a_{i_1,\ldots,i_d}>0\},\\
\log\beta(Y)=\max\{\log a_{i_1,\ldots,i_d} +\sum_{j=1}^d y_{i_j}, i_1,\ldots,i_d\in[n], a_{i_1,\ldots,i_d}>0\}
\end{eqnarray*}
Assume furthermore that $\cA>0$.  Then $\mathbf{C}=\mathbf{B}(P,\0)$.
\end{lemma}
\begin{proof}
Let $\mathbf{W}(\cA)\subset \otimes^d\R^n$ be the subspace of all $\cX=[x_{i_1,\ldots,i_d}]$ such that $x_{i_1,\ldots,i_d}=0$ if $a_{i_1,\ldots,i_d}=0$.  Denote by $\pi_\cA: \otimes^d\R^n\to \mathbf{W}(\cA) $ the orthogonal projection on $\mathbf{W}(\cA) $.
Let $r_\cA:\mathbf{W}(\cA)\to \R$ be given by 

\noindent
$r_\cA=\sum_{i_1,\ldots,i_d\in[n]} a_{i_1,\ldots,i_d}\exp(y_{i_1,\ldots,i_d})$.  The Hessian $He(r_\cA)(\cY)$ of $r_\cA$ at $\cY$ is a diagonal matrix with entries $a_{i_1,\ldots,i_d}\exp(y_{i_1,\ldots,i_d}), a_{i_1,\ldots,i_d}>0$.  Therefore the maximum and the minimum eigenvalue of $He(r_\cA)(\cY)$ are
\begin{eqnarray*}
&&\lambda_{\max}(\cY)=\max\{a_{i_1,\ldots,c_d}\exp(y_{i_1,\ldots,i_d}), a_{i_1,\ldots,i_d}>0\} \ge \\
&&\lambda_{\min}(\cY)=\min\{a_{i_1,\ldots,c_d}\exp(y_{i_1,\ldots,i_d}), a_{i_1,\ldots,i_d}>0\}.
\end{eqnarray*}
Hence $r_\cA$ is strictly convex on $\mathbf{W}(\cA)$.  Let $\psi:(\R^n)^d\to \otimes^d\R^n$ be given by
\begin{eqnarray*}
(\psi(\y_1,\ldots,\y_d))_{i_1,\ldots,i_d}=\sum_{j\in [d]} y_{i_j,j}, \forall i_1,\ldots,i_d\in[n].
\end{eqnarray*}
By definition of $\mathbf{C}$ the map $\pi_\cA\circ \psi: \mathbf{C}\to \mathbf{W}(\cA)$ is an injection. Let $\mathbf{W}=\pi(\psi(\mathbf{C}))$.  Observe next that $g_\cA|\mathbf{C}=r_\cA|\mathbf{W}$.  Hence $g_\cA|\mathbf{C}$ is strictly convex on  $\mathbf{C}$.  The Cauchy interlacing theorem yields that the eigenvalues of the Hessian $He(r_\cA)(\cY), \cY\in \mathbf{W}$  restricted to the subspace $\mathbf{W}$ are in the interval $[\beta(\y_1,\ldots,\y_d),\alpha(\y_1,\ldots,\y_d)]$.

Assume now that $\cA>0$.  Suppose that $\psi(\cA)(\y_1,\ldots,\y_d)=0$:
\begin{eqnarray*}
y_{i_1,1}+\cdots+y_{i_d,d}= 0 \, \forall i_1,\ldots,i_d\in[n].
\end{eqnarray*}
Fix $k\in[d]$ and sum the above equalities on all $i_j$ for $j\ne d$.  Hence $\y_k=t_k\1$, and $\sum_{k\in[d]} t_k=0$.  Clearly $\langle \bp_k,\y_k\rangle=t_k h$.  Hence $\psi:\mathbf{B}(P,\0)\to \otimes^d\R^n$ is an injection.  Thus $\mathbf{C}=\mathbf{B}(P,\0)$.
\end{proof}

In what follows we assume that $\cA$ satisfies one of the equivalent assumptions of Lemma \ref{rminlem}.  
\subsection{Matrix scaling.}\label{subsec:mat}
The classical studied case is matrix scaling ($d=2$): \cite{ALOW17F,BPS66,KK96,KLRS08,Men68,MS69,NR99,Sin64,SK67}.
Let $A_0=A,\ldots, A_k,\ldots,$ be matrices obtained by any convergent scaling algorithm.  
One of the most used algorithms in matrix scaling is Sihkhorn algorithm \cite{Sin64}, sometimes abbreviated here as SA, also known as the RAS algorithm \cite{Bac70}.
Let us describe it briefly for the square matrices.   Assume that $A=[a_{i,j}]\in\R_+^{n\times n}$ is scalable to $U\in\rU(P)$.
Perform the row and column scaling alternatively:
\begin{eqnarray*}
R(A)=[\frac{p_{i,1}}{r_i}a_{i,j}], \quad C(A)=[\frac{p_{j,2}}{c_j}a_{i,j}], \quad \br=A \1,\; \bc=A^\top \1.
\end{eqnarray*}
Then $A_{2k-1}=R(A_{2k-2}), A_{2k}=C(A_{2k-1})$ for $k\in\N$.

We now recall a well known fact that Sinkhorn algorithm is a partial minimization algorithm \cite{KLRS08}.  Consider the function $g_A$:
\begin{eqnarray*}
g_A(\x,\y)=\sum_{i,j=1}^n e^{x_i+y_j}a_{i,j} -\sum_{i=1}^n p_{i,1}x_i -\sum_{j=1}^n p_{j,2}y_j,
\end{eqnarray*}
Find $\min\{g_A(\x,0), \x\in\R^n\}$ and $\min\{g_A(0,\y), \y\in\R^n\}$.  These minima are achieved at the unique  critical points  
\begin{eqnarray*}
&&\x^{(1)}=(x^{(1)}_1,\ldots,x_n^{(1)}),\quad  x_i^{(1)}=\log \frac{p_{i,1}a_{i,j}}{r_i},\\  
&&\y^{(1)}=(y^{(1)}_1,\ldots,y_n^{(1)}),\quad y_j^{(1)}=\log \frac{p_{j,2}a_{i,j}}{c_j}.
\end{eqnarray*}
Note that $g_A(\x^{(1)},\y)=g_{R(A)}(0,\y)$ and $g_A(\x,\y^{(1)})=g_{C(A)}(\x,0)$.
Define $\x(\y)\defeq\arg\min f(\x,\y)$, where $\y$ is fixed, and $\y(\x)\defeq\arg\min f(\x,\y)$, where $\x$ is fixed.  Then Sinkhorn algorithm is the following partial minimization algorithm: 
Start with $\x^{(0)}=\y^{(0)}=0$.  Then  for $k\in\N$
\begin{eqnarray*}
\begin{cases}
\x^{(k)}=\x^{(k-1)}+\x(\y^{(k-1)}), &\y^{(k)}=\y^{(k-1)}  \textrm{ for } k \textrm{ odd},\\
\x^{(k)}=\x^{(k-1)},& \y^{(k)}=\y^{(k-1)}+\y(\x^{(k)})  \textrm{ for } k \textrm{ even}
\end{cases}
\end{eqnarray*}

The following complexity results are known for Sinkhorn algorithm.  We assume that $A\in\Q_+^{n\times n}, \bp_1,\bp_2\in \Q_{++}^n, \|\bp_1\|_1=\|\bp_2\|_1=h$. 
In some complexity estimates we will assume without loss of generality that $\bp_1,\bp_2\in\Z_{++}^n$.  Let $A_{k}=X_{k} A Y_{k}$ for $k\in\N$ where 
\begin{eqnarray*}
\begin{cases}
A_{2k-1}=X_{2k-1} A Y_{2k-1}, & X_{2k-1}=\diag(\exp(\x^{(k)}), Y_{2k-1}=\diag(\exp(\y^{(k-1)}), \\
A_{2k}=X_{2k} A Y_{2k}, & X_{2k}=\diag(\exp(\x^{(k)}), Y_{2k}=\diag(\exp(\y^{(k)})
\end{cases}
\end{eqnarray*}
Here $\diag(\y)=\diag(y_1,\ldots,y_n)$   For a given $s\in[1,\infty]$ and $\varepsilon \in(0,1/2)$ we stop our algorithm when
\begin{eqnarray}\label{matsap}
\max\bigl(\|A_{k} \1 -\bp_1\|_s,\|A_k^\top \1-\bp_2\|_s\bigr)\le \varepsilon.
\end{eqnarray}
Observe that the left hand side of this inequality is a decreasing function of $s\in[1,\infty]$.  Furthermore, at least one of the terms appearing in $\max$ is zero.

Assume first that $A$ is positive.  Then for $\bp_1,\bp_2\in \Z_{++}^n$ the complexity of Sinkhorn algorithm for $s=2$ is $O(\frac{\rho h}{\varepsilon^2}\log \frac{h\theta}{\eta})$ \cite[Theorem 4.1]{KLRS08}, where $\rho$ is the maximum entry of $\bp_1$ and $\bp_2$.
A stronger result is obtained in \cite{AWR17}.  
Let $\|A\|_1$ be the sum of all entries of $A$.   Then there exists $k\in [ \lceil\frac{4\log (\|A\|_1/\eta)}{\varepsilon^2}\rceil]$ such that \eqref{matsap} is satisfied for $s=1$.  

If $A$ is nonnegative then the complexity of Sinkhorn algorithm with respect to the number of iterations  is $O(\frac{ h^2}{\varepsilon^2}\log \frac{\rho\theta}{\eta})$ \cite[Theorem 5.2]{KLRS08} for $s=2$.  Better complexity results for different algorithms are given in \cite{ALOW17F}.
Another complexity estimate using an ellipsoid method for convex optimization methods for is
$\widetilde {O}\bigl(n^4 \log(1/\varepsilon) \log (\theta/\eta)\bigr)$, (ignoring the $\log$ factors) for $s=\infty,\bp_1=\bp_2=\1$ \cite{KK96}  and $s=1$ \cite{NR99}.

The big advantage of Sinkhorn method are: simplicity of implementation, possible implementation in parallel, simplicity in analyzing the distance given by the left hand side \eqref{matsap}.  The disadvantage to compare with other methods is the factor $\varepsilon^{-2}$ versus $\log (1/\varepsilon)$ of other methods.  
\subsection{Tensor scaling for positive tensors}\label{subsec:tenscal1}
For $\cA=[a_{i_1,\ldots,i_d}]\in\otimes^d\R^n$ denote 
\begin{eqnarray*}
&&s_l(\cA)=(s_{1,l}(\cA),\ldots,s_{n,l}(\cA)), \quad s_{i_l,l}(\cA)=\sum_{i_j\in[n], j\ne l}c_{i_1,\ldots,i_l,\ldots,i_n}, \; l\in[d],\\
&&\|\cA\|_\infty=\max\{|a_{i_1,\ldots,i_d}|, i_1,\ldots,i_d\in[n]\}.
\end{eqnarray*}
In this paper we assume that $\cA\in\otimes^d\R^n_+$, where $\s_j(\cA)>\0$ for $j\in [d]$, unless stated otherwise.
Given a positive vector $\br=(r_1,\ldots,r_n)$ and $j\in[d]$ we denote by $\cD(\cA,\br,j)$ the following (diagonal) scaling of $\cA$ in the $j$-mode:
\begin{eqnarray}\label{diagscal}
\cD(\cA,\br,j)=[\frac{r_{i_j}}{s_{i_j,j}}a_{i_1,\ldots,i_j,\ldots,i_d}].
\end{eqnarray}
Then $\s_j(\cD(\cA,\br,j))=\br$.
It is straightforward to show that
\begin{eqnarray}\label{basinSA}
\|\cD(\cA,\br,j)-\cA\|_1=\|\br-\s_j(\cA)\|_1.
\end{eqnarray}
For $\x,\y\in \R^n_+$ we define $\x\wedge \y=(\min(x_1,y_1),\ldots,\min(x_n,y_n))$.

In this subsection we show that for a positive tensor $\cA\in\otimes^d\R^n$  Sinkhorn algorithm converges linearly to the unique $\cU\in\rU(\bp_1,\ldots,\bp_d)$. 

Let $\x_l(\cA)=(x_{1,l}(\cA),\ldots,x_{n,l}(\cA))$ be given by
\begin{eqnarray}\label{xilcCform}
x_{i_l,l}(\cA)=\log p_{i_l,l}-\log s_{i_l,l}(\cA), \quad i_l\in[n], \; l\in[d].
\end{eqnarray}
Recall the Kullback-Leibler divergence between $\bp,\bq\in\Pi^n$:
\begin{eqnarray*}
\cK(\bp||\bq)=\sum_{i=1}^n p_i(\log p_i -\log q_i).
\end{eqnarray*}
Assume that $\bp_1,\ldots,\bp_d)\in\Pi^n_o$.  Then
\begin{eqnarray}\label{prodCpform}
\langle \bp_j,\x_j(\cA)\rangle=\cK\bigl(\bp_j,\|\cA\|_1^{-1}\s_j(\cA)\bigr) -\log\|\cA\|_1.
\end{eqnarray}

We define  Sinkhorn algorithm as follows.  
Start with $\cA_0=\cA$ and the initial point 
$X^{(0)}=(\x_1^{(0)},\ldots,\x_d^{(0)})=\0$.  Suppose that $X^{(k-1)}=(\x_1^{(k-1)},\ldots,\x_d^{(k-1)})$ is defined for $k\in\N$.  Denote
\begin{eqnarray}\label{defAkSA}
\cA_{k-1}=[a_{i_1,\ldots,i_d,k-1}]=[\exp(\sum_{j\in[d]}x^{(k-1)}_{i_j,j})a_{i_1,\ldots,i_d}].
\end{eqnarray}
Then  
\begin{eqnarray}\label{defellk}
&&\ell(k-1)=\arg\max\{\|\s_j(\cA_{k-1})-\frac{\langle \s_j(\cA_{k-1}),\bp_j\rangle}{\|\bp_j\|^2} \bp_j\|_1,j\in[d]\},\\  \label{defSA}
&&\x^{(k)}_j=
\begin{cases}
\x^{(k-1)}_j &\textrm{ for } j\ne l(k-1),\\
\x_j(\cA_{k-1})+\x_j^{(k-1)} &\textrm{ for } j=l(k-1)
\end{cases}
\end{eqnarray}
Note that for $d=2$  the above Sinkhorn algorithm is different from the standard Sinkhorn algorithm only for $k=1$.  Clearly for $k\ge 1$ the tensor $\cA_k$ is a probability tensor.  Furthermore, $\cA_k=\cD(\cA_{k-1},\bp_{\ell(k-1)},\ell(k-1))$.

The main result of this subsection is:
\begin{theorem}\label{linconvSA}  Let $\bp_1,\ldots,\bp_d\in \Pi^n_o$.  Assume that the tensor $\cA\in\otimes^d\R^n$ is positive.  
Then 
\begin{enumerate}[\upshape (a)]
\item Sinkhorn algorithm converges geometrically to a unique $\cU\in\rU(P)$.  
\item The sequence $\cK(\bp_j, \s_j(\cA_k)), k\in\N$ converges geometrically to $0$ for each $j\in[d]$.
\item The sequence $X^{(k)})$ converge geometrically to a minimum point $X^\star$ of $g_\cA$.
\item For $k\ge 1$ we have
\begin{eqnarray}\label{difxkxk}\begin{split}
&g_\cA(X^{(k)})-g_\cA(X^{(k+1)})=\\
&\sum_{j=1}^d\langle \bp_j,\x_j^{(k+1)}-\x_{j}^{(k)}\rangle=\langle \bp_l,\x_l^{(k+1)}-\x_{l}^{(k)}\rangle=\cK(\bp_l||\s_l(\cA_{k})), 
\end{split}
\end{eqnarray}
for  $l=\ell(k)$.
\end{enumerate}
\end{theorem}
\begin{proof} (a) 
Let $\mathbf{V}=\mathbf{B}(P,\0)$.  Lemma \ref{eigHesest} yields that $g_{\cA}$ is strictly convex on $\mathbf{V}$.  
(Recall that $g_\cA|\mathbf{V}=\tilde g_\cA| \mathbf{V}$.)
Since $\cA$ is  positive it is scalable to a unique $\cU\in\rU(P)$ Lemma \eqref{rminlem} yields that $g_\cA|\mathbf{V}$ has minimum.  Theorem \ref{convalgo} yields that PM algorithm for $g_\cA$ on $\mathbf{V}$ converges geometrically to the unique minimum point $\tilde X^\star\in(\R^n)^d$.  (Here $\mathbf{V}_j=\{\x=(\0,\ldots,\0,\x_j,\0,\ldots,\0)\in \mathbf{V}\}$.)

We now compare Sinkhorn algorithm to PM algorithm.  For a nonzero vector $\bu\in\R^n$ denote by $\pi_\bu$ the orthogonal projection of $\x\in\R^n$ on the line spanned by $\bu$:
\begin{eqnarray*}
\pi_\bu(\x)=\frac{1}{\|\bu\|^2}\langle \x,\bu\rangle \bu.
\end{eqnarray*}
Assume that $\phi:(\R^n)^d\to \R$ is a continuous function with first continuous derivatives on $(\R^n)^d$.  Let  
$ \widetilde{\nabla} _{\x_j}\phi(\y_1,\ldots,\y_d)$ be the gradient of $\phi$ at $(\y_1,\ldots,\y_d)$ in the direction of the subspace $\langle \bp_j,\z_j\rangle=0$.  We claim that $ \widetilde{\nabla} _{\x_j}\phi(Y)$ is the orthogonal projection of $ {\nabla} _{\x_j}\phi(Y)$
on the orthogonal complement of the line spanned by $\bp_j$:
\begin{eqnarray*}
 \widetilde{\nabla} _{\x_j}\phi(Y)=\nabla _{\x_j}\phi(Y)-\pi_{\bp_j}(\nabla _{\x_j}\phi(Y)).
\end{eqnarray*}
This formula holds if $\phi$ is a linear function.  Hence it hold for general $\phi$.

We start PM algorithm with $\tilde X^{0)}=(\tilde\x^{(0)}_1,\ldots,\tilde\x^{(0)}_d)=\0$.  
Let $\tilde X^{(k)}=(\tilde\x^{(k)}_1,\ldots,\tilde\x^{(k)}_d)$ be the $k$-th step in PM algorithm.
Denote 
\begin{eqnarray*}
\widetilde {\cA}_k=[\tilde a_{i_1,\ldots,i_d,k-1}]=[\exp(\sum_{j\in[d]} \tilde x^{(k)}_{i_j,j})a_{i_1,\ldots,i_d}], \quad k=0,1,\ldots .
\end{eqnarray*}
Observe that $\|\widetilde {\cA}_k\|_1=\tilde g_{\cA}(\tilde\x^{(k)}_1,\ldots,\tilde\x^{(k)}_d)$.
We claim that 
\begin{eqnarray}\label{kstepPMA}
\quad\quad\tilde\x_j^{(k)}=
\begin{cases}
\tilde\x^{(k-1)}_j &\textrm{ for } j\ne l(k-1),\\
\tilde\x^{(k-1)}_j+\x_j(\widetilde{\cA}_{k-1} )+&
(-\cK\bigl(\bp_j,\|\widetilde{\cA}_{k-1}\|_1^{-1}\s_j(\widetilde{\cA}_{k-1})\bigr)+\\ \log\|\widetilde{\cA}_{k-1}\|_1)\1 &\textrm{ for } j= l(k-1).
\end{cases}
\end{eqnarray}
Here $l=\ell(k-1)$ and $\x_l(\widetilde{\cA}_{k-1} )$ is given by \eqref{xilcCform}.

It is enough to show the above formula for $k=1$.  Note that $\widetilde{\cA}_0=\cA$.
Then the choice of choosing $j\in[d]$ for which $ \|\widetilde{\nabla} _{\x_j}\tilde g_{\cA}(0)\|_s$ is maximum is $l(0)$ as in Sinkhorn algorithm.  Recall that 
\begin{eqnarray*}
\widetilde{\nabla}_{\x_j}\tilde g_{\cA}(\0,\ldots,\0,\y_j,\0,\ldots,\0)={\nabla}_{\x_j} \tilde g_{\cA}(\0,\ldots,\0,\y_j,\0,\ldots,\0)-t\bp_j
\end{eqnarray*}
for the value of $t$ such that $\langle\widetilde{\nabla}_{\x_j}\tilde g_{\cA}(\0,\ldots,\0,\y_j,\0,\ldots,\0),\bp_j\rangle=0$.  Assume that $\widetilde{\nabla}_{\x_j}\tilde g_{\cA}(\0,\ldots,\0,\y_j,\0,\ldots,\0)=\0$.
As ${\nabla}_{\x_j}\tilde g_{\cA}(\0,\ldots,\0,\y_j,\0,\ldots,\0)>\0$ it follows that $t>0$.  Hence
$\tilde\x_l^{(1)}=\x_l(\cA)+\log t\1$.  Combine the assumption that $\langle \tilde\x_l^{(1)},\bp_l\rangle=0$ with \eqref{prodCpform} to deduce \eqref{kstepPMA} for $k=1$.  Use \eqref{kstepPMA} to deduce
\begin{eqnarray}\label{tildeCkCkform}
&&\widetilde{\cA}_k=a_{k-1}
[\exp(x^{(k)}_{i_l,l}(\widetilde{\cA}_{k-1}))
\tilde {a}_{i_1,\ldots,i_d,k-1}]\\
&&a_{k-1}=\exp(-\cK\bigl(\bp_l,\|\widetilde{\cA}_{k-1}\|_1^{-1}\s_l(\widetilde{\cA}_{k-1})) +\log\|\widetilde{\cA}_{k-1}\|_1).\notag
\end{eqnarray}
Theorem \ref{convalgo} yields that $\tilde X^{(k)}$ converge geometrically to the minimum point $\tilde X^\star=(\tilde x^{\star}_1,\ldots, \tilde x^{(\star)}_d)$.  Hence the convergence of $\widetilde{\cA}_k $ to $\widetilde{\cA}_\star$ and the convergence \begin{eqnarray*}
\lim_{k\to\infty}\tilde g_\cA(\tilde X^{(k)})=\lim_{k\to\infty}\|\widetilde \cA_k\|_1=
\tilde g_\cA(\tilde X^{\star})= \|\widetilde \cA_\star\|_1
\end{eqnarray*}  
is geometrical.
As $\tilde X^\star$ is a critical point of  $\tilde g_\cA|\mathbf{V}$
it follows that 
\begin{eqnarray*}
{\nabla}_{\x_j} \tilde g_{\cA}(\tilde X^\star)=t_j\bp_j, \quad j\in [d].
\end{eqnarray*}
Hence $t_j=\|\widetilde{\cA}_\star\|_1$ and $\widetilde{\cA}_\star\in \rU(\|\widetilde{\cA}_\star\|_1P)$.

Let $\cA_\star\defeq\|\widetilde {\cA}_\star\|_1^{-1} \widetilde {\cA}_\star$.  Then $\cA_\star\in\rU(P)$ . We claim that $\cA$ is scalable to $\cA_\star$.  Furthermore the sequemce $\cA_k, k\in\N$ given by Sinkhorn algorithm converges geometrically to $\cA_\star$.

We first observe that in each step of Sinkhorn algorithm    we have equality
$\cA_k=r_k\widetilde{\cA}_k$ for $k\ge 0$.  We prove this claim by induction on $k\ge 0$.  Clearly $r_0=1$.    Suppose that $\cA_m=r_m\widetilde{\cA}_m$.  Observe that
$\tilde\g_{\cA_m}(\y_1,\ldots,\y_d)=r_m \g_{\widetilde{\cA}_m}(\y_1,\ldots,\y_d)$.
Hence
\begin{eqnarray*}
\nabla_{\x_j} \tilde g_{\cA_m}(\0,\ldots,\0,\y_j,\0,\ldots,\0)=&r_m\nabla_{\x_j} \tilde g_{\widetilde{\cA}_m}(\0,\ldots,\0,\y_j,\0,\ldots,\0),\\
\pi_{\bp_j}(\nabla_{\x_j} \tilde g_{\cA_m}(\0,\ldots,\0,\y_j,\0,\ldots,\0))=&r_m\pi_{\bp_j}(\nabla_{\x_j} \tilde g_{\widetilde{\cA}_m}(\0,\ldots,\0,\y_j,\0,\ldots,\0)).
\end{eqnarray*}
The above equalities show that the choice $l=\ell(m)$ is identical for both algorithms.  
Clearly 
\begin{eqnarray*}
\x_j(\cA_m)=\x_j(r_m\widetilde{\cA}_m)= \x_j(\widetilde{\cA}_m)-\log r_m \1,
\end{eqnarray*}
Next observe that Sinkhorn algorithm give the relation

\noindent
$\cA_{m+1}= [\exp(x_{i_l,l})a_{i_1,\ldots,i_m,m}]$. Compare this equaility with \eqref{tildeCkCkform} to deduce
$\cA_{m+1}=r_{m+1}\widetilde{\cA}_{m+1}$.  Observe next that for $m\ge 1$ that $\cA_m$ is a probability tensor.  Hence
\begin{eqnarray}\label{conCtildeC}
\cA_{m}=\frac{1}{\|\widetilde{A}_m\|_1} \widetilde{A}_m.
\end{eqnarray}
As $ \widetilde{A}_m$ converges geometrically to $\widetilde {\cA}_\star$ it follows that $\cA_m$ converges geometrically $\cA_\star$.  

\noindent (b)
As the sequence $\cA_k,k\in\N$ converges geometrically of $\cU\in \rU(P)$ it follows that the sequence $\s_j(\cA_k),k\in\N$ converges geometrically to $\bp_j$.  Hence the sequence $s_{i,j}(\cA_k),k\in\N$ converges geometrically to $p_{i,j}>0$.  Therefore the sequence $\log (s_{i,j}(\cA_k)/p_{i,j}),k\in\N$ converges geometrically to $0$.  This yields that the sequence $\cK(\bp_j,\s_j(\cA_k)), k\in\N$ converges geometrically to zero.

\noindent (c)
As $\tilde X^{(k)}$ converge geometrically to $\tilde X^\star$ we deduce that  $\tilde X^{(k+1)}-\tilde X^{(k)}$ converge geometrically to zero vector.  Assume that $k\ge 1$.  Then $\cA_{k}=\|\widetilde{\cA}_{k}\|_1^{-1} \widetilde{\cA}_{k}$.  Hence $\s_l(\|\widetilde{\cA}_{k}\|_1^{-1}\widetilde{\cA}_{k})=\s_l(\cA_{k}) $.  Furthermore $\x_{l}(\cA_k)=\x_l(\widetilde{\cA}_k)+\log \|\widetilde{\cA}_{k}\|_1\1$.  Use \eqref{kstepPMA}
to deduce
\begin{eqnarray*}
\begin{cases}
\tilde\x^{(k+1)}_j-\tilde\x^{k}_j=0 &\textrm{ for } j\ne \ell(k),\\
\tilde\x^{(k+1)}_j-\tilde\x^{k}_j=\x_j(\cA_k)-\cK(\bp_j,\s_j(\cA_k))\1 &\textrm{ for } j= \ell(k).
\end{cases}
\end{eqnarray*}
Part (b) of the theorem yields that the sequence $\cK(\bp_j,\s_j(\cA_k)), k\in\N$ converges geometrically to zero.  Hence the sequence $\x_{\ell(k)}(\cA_k)$ converges to zero vector geometrically.  From the definition of Sinkhorn algorithm it follows that the sequence $X^{(k+1)}-X^{(k)}$  converge geometrically to zero vector.  Define 
$X^\star=\sum_{k=0}^\infty \bigl(X^{(k+1)}-X^{(k)}\bigr)$.
Then $X^\star$ is a critical point of $g_\cA$. 
Furthermore the sequence $X^{(k)}$ converges geometrically to $X^\star$.

\noindent (d)
Recall the formulas \eqref{defgC} and \eqref{deftildegC}.  As $\cA_k$ is a probability vector it follows that $\tilde g_\cA(X^{(k)})=1$ for $k\ge 1$.  This yields the first equality in \eqref{difxkxk}. The definition of Sinkhorn algorithm yields the second equality of \eqref{difxkxk} for $l=\ell(k)$. The third equality in \eqref{difxkxk} follows from \eqref{xilcCform}.
\end{proof}
\subsection{Estimates of convergence of SA for positive tensors}
In this subsection we generalize the estimates for matrices in \cite{AWR17} to tensors.
We assume in this subsection that $\bp_1,\ldots,\bp_d\in\Pi^n_o$.
 We let $\cA_0=\frac{1}{\|\cA\|_1}\cA$ to be a probability tensor.
For a given $\varepsilon\in (0,1/2)$ we stop our algorithm the first time we obtain
\begin{eqnarray}\label{stoptalg}
\max\{\|\s_j(\cA_{k})-\frac{\langle \s_j(\cA_{k}),\bp_j\rangle}{\|\bp_j\|^2} \bp_j\|_1,j\in[d]\}<\varepsilon.
\end{eqnarray}
We denote by $k^{stop}\ge 0$ the first $k\ge 0$ for which the above inequality holds.
Lemma \ref{pqdistest} yields that
\begin{eqnarray}\label{distsAkp}
\|\s_j(\cA_{k^{stop}})-\bp_j\|_1<2\varepsilon \textrm{ for } j\in[d].
\end{eqnarray}
\begin{theorem}\label{stopest}  Assume $\bp_1,\ldots,\bp_d\in\Pi^n_o$, and  $\cA\in\otimes^d\R^n$ is a positive  tensor with the minimum entry equal to $\beta$. 
Let $\varepsilon\in (0,1/2)$ be given. Then the Sinkhorn algorithm  given by  \eqref{defellk}-\eqref{defSA} with $\cA_0=\frac{1}{\|\cA\|_1}\cA$ and the stopping criterion \eqref{stoptalg} satisfies $k^{stop}\le \frac{2(\sqrt{n}+1)^2}{\varepsilon^2}\log \frac{\|\cA\|_1}{\eta}$.
\end{theorem}
\begin{proof} Observe that  $\eta_0=\frac{\eta}{\|\cA\|_1}$ is the minimum element of $\cA_0$.  Thus $\frac{\|\cA\|_1}{\eta}=\frac{1}{\eta_0}$.
Note that if in the Sinkhorn algorithm we replace $\cA$ by $\cA'=\cA_0$ then $\cA_k=\cA'_k$ for $k\ge 1$.  Thus to prove the theorem we can assume without loss of generality that $\cA$ is a positive probability tensor, i.e., $\cA=\cA_0$.
The assumption $\|\cA\|_1=1$ yields that \eqref{difxkxk} holds for $k=0$.  By Theorem \ref{linconvSA} $k^{stop}<\infty$.  
Note that for $k=0,\ldots,k^{stop}-1$ we have that 
\begin{eqnarray*}
&&\max\{\|\s_j(\cA_{k})-\frac{\langle \s_j(\cA_{k}),\bp_j\rangle}{\|\bp_j\|^2} \bp_j\|_1,j\in[d]\}=\\
&&\|\s_{\ell(k)}(\cA_{k})-\frac{\langle \s_{\ell(k)}(\cA_{k}),\bp_j\rangle}{\|\bp_j\|^2} \bp_j\|_1\ge\varepsilon.
\end{eqnarray*}
Combine the equality \eqref{difxkxk} with Lemma \ref{pqdistest} to deduce 
\begin{eqnarray*}
g_\cA(X^{(k)})-g_\cA(X^{(k+1)})=
\cK(\bp_{\ell(k)}||\s_{\ell(k)}(\cA_{k}))\ge \frac{\varepsilon^2}{2(\sqrt{n}+1)^2}.
\end{eqnarray*}
Hence
\begin{eqnarray*}
g_\cA(X^{(0)})-g_\cA(X^{(k^{stop}+1)})\ge \frac{k^{stop}\varepsilon^2}{2(\sqrt{n}+1)^2},
\end{eqnarray*}
for $k\le k^{stop}$.
Sum up the equalities in \eqref{difxkxk} from $k=0$ to $k=m$ to deduce
\begin{eqnarray*}
g_\cA(X^{(0)})-g_\cA(X^{(m+1)})=
\sum_{j=1}^d\langle \bp_j,\x_j^{(m+1)}\rangle.
\end{eqnarray*}
Letting $m\to\infty$ we obtain $g_\cA(X^{(0)})-g_\cA(X^\star)=
\sum_{j=1}^d\langle \bp_j,\x_j^{\star}\rangle$.
Recall that 
\begin{eqnarray*}
1=\tilde g_\cA(X^\star)=\sum_{i_j\in[d],j\in[d]}\exp(\sum_{l=1}^d x_{i_l,l}^{\star})a_{i_1,\ldots,i_d}.
\end{eqnarray*}
Hence
\begin{eqnarray*}
\exp(\sum_{l=1}^d x_{i_l,l}^{\star})a_{i_1,\ldots,i_d}\le 1\Rightarrow \sum_{l=1}^d x_{i_l,l}^{\star}\le \log (1/\eta).
\end{eqnarray*}
Multiply the last inequality by $p_{i_1,1}\cdots p_{i_d,d}$ and sum on $i_j\in[d]$ for $j$ in $[d]$ to obtain $\sum_{j=1}^d\langle \bp_j,\x_j^{\star}\rangle\le  \log (1/\eta)$.
Hence
\begin{eqnarray*}
\log (1/\eta)\ge g_\cA(X^{(0)})-g_\cA(X^\star)>
g_\cA((X^{(0)})-g_\cA(X^{(k^{stop}+1)})\ge \frac{k^{stop}\varepsilon^2}{2(\sqrt{n}+1)^2}.
\end{eqnarray*}
\end{proof}
\subsection{Approximating TOT value}
The SA returns positive tensor $\cA_{k^{stop}}$ which usually will not by in $\rU(P)$.  We need to replace $\cA_{k^{stop}}$ by a tensor $\cB\in \rU(P)$.  We start with the following simple lemma whose proof is left to the reader:
\begin{lemma}\label{adlem} Let $\bp_1,\bq_1,\ldots,\bp_d,\bq_d\in\R^n_{++}$ satisfying the conditions:
\begin{eqnarray*}
\bp_j\ge \bq_j \textrm{ for } j\in [d], \|\bp_1\|_1=\ldots=\|\bp_d\|_1=h\ge\|\bq_1\|_1=\ldots=\|\bq_d\|_1=h'.
\end{eqnarray*}
Assume that $\cG=\rU(Q), Q=(\bq_1,\ldots,\bq_d)$.  Then
\begin{eqnarray*}
&&\cB=\cG+\frac{1}{(h-h')^{d-1}}\otimes_{j=1}^d(\bp_j-\bq_j)\in \rU(\bp_1,\ldots,\bp_d),\\
&&\|\cB-\cG\|_1=\|\cB\|_1-\|\cG\|_1=h-h'.
\end{eqnarray*}
\end{lemma}
The following algorithm is a generalization of \cite[Algorithm 2]{AWR17} which we state as a lemma:
\begin{lemma}[Replacement Algorithm]\label{repalg}
Assume that
\begin{eqnarray*}
&&\bp_1,\ldots,\bp_d,\br_1,\ldots,\br_d\in\R^n_{++}, \\
 &&\|\bp_1\|_1=\ldots=\|\bp_d\|_1=h,\quad\|\br_1\|_1=\ldots=\|\br_d\|_1=h''.
\end{eqnarray*}
Let $\cF\in\rU(R), R=(\br_1,\ldots,\br_d)$.  Set $\cF_0=[f_{i_1,\ldots,i_d,0}]=\cF$.  Define $\cF_j=[f_{i_1,\ldots,i_d,j}]\in \otimes\R_+^n$ for $j\in[d]$, a scaling of $\cF$,  recursively as follows: 
\begin{eqnarray*}
f_{i_1,\ldots,i_d,j}=\min(\frac{p_{i_j,j}}{s_{i_j,j}(\cF_{j-1})},1)f_{i_1,\ldots,i_j,\ldots,i_d,j-1}, \quad i_1\in[n],\ldots,i_d\in[n].
\end{eqnarray*}
Let $\cG=\cF_d$ and $\bq_j=\s_j(\cA)$ for $j\in [d]$.  Then the assumptions of Lemma \ref{adlem} holds.  Let $\cB$ be the tensor defined in  Lemma \ref{adlem}.
Then $\cB$ is obtained from $\cF$ in $O(dn^d)$ operations.  Furthermore
\begin{eqnarray*}
\|\cB-\cF\|_1\le 2\sum_{j=1}^d\|\bp_j-\br_j\|_1.
\end{eqnarray*}
\end{lemma}
\begin{proof}  Let $j\in[d]$.  Clearly $\cF_j=\cD(\cF_{j-1}, \bp_j\wedge \s_j(\cF_{j-1}),j)$.  Furthermore 
\begin{eqnarray*}
\cF_j\le \cF_{j-1}, \quad \s_j(\cF_j)= \bp_j\wedge  \s_j(\cF_{j-1})\le\bp_j, \quad j\in[d].
\end{eqnarray*}
As $\cF_j\le \cF$ it follows that $\s_l(\cF_j)\le \br_l$ for $l,j\in[d]$.  Furthermore $\s_l(\cG)\le \bp_l\wedge \br_l$.  Hence $\cG$ satisfies the assumptions of Lemma \ref{adlem}.
Let us compute and estimate
\begin{eqnarray*}
\Delta=\|\cG-\cF\|_1 =\|\cF\|_1-\|\cG\|_1=h''-h'=\sum_{j=1}^d(\|\cF_{j-1}\|_1 -\|\cF_j\|_1).
\end{eqnarray*}
Recall $x_+=\max(0,x)$ for $x\in\R$.  Hence $2x_+=|x|+x$. Thus for $\x,\y\in \R^n_+$ we have the equality and the inequality 
\begin{eqnarray*}
\sum_{i=1}^n (x_i - y_i)_+=\frac{1}{2}\bigl(\|\x-\y\|_1 +\|\x\|_1-\|\y\|_1\bigr)\le \|\x-\y\|_1.
\end{eqnarray*}
As $ \s_j(\cF_j)= \bp_j\wedge  \s_j(\cF_{j-1})$ we deduce that 
\begin{eqnarray*}
\|\cF_{j-1}\|_1 -\|\cF_j\|_1=\sum_{i=1}^n (s_{i,j}(\cF_{j-1})-p_{i,j})_+.
\end{eqnarray*}
Hence
\begin{eqnarray*}
\|\cF_0\|_1-\|\cF_1\|_1=\sum_{i=1}^n (s_{i,1}(\cF_{j-1})-p_{i,1})_+=\frac{1}{2}\bigl(\|\br_1-\bp_1\|_1+h''-h\bigr).
\end{eqnarray*}
Thus
\begin{eqnarray*}
\Delta=\frac{1}{2}\bigl(\|\br_1-\bp_1\|_1+h''-h\bigr)+\sum_{j=2}^d\sum_{i=1}^n (s_{i,j}(\cF_{j-1})-p_{i,j})_+.
\end{eqnarray*}
For $j\ge 2$ we observe
\begin{eqnarray*}
\|\cF_{j-1}\|_1-\|\cF_j\|=\sum_{i=1}^n (s_{i,j}(\cF_{j-1})-p_{i,j})_+\le \sum_{i=1}^n (r_{i,j}-p_{i,j})_+\le \|\br_j -\bp_j\|_1.
\end{eqnarray*}
Thus
\begin{eqnarray*}
&&\|\cB-\cF\|_1\le\|\cB-\cG\|_1+\|\cG-\cF\|_1=h-h'+\Delta= h-h''+2\Delta=\\
&&\|\br_1-\bp_1\|_1+2\sum_{j=2}^d\sum_{i=1}^n (s_{i,j}(\cF_{j-1})-p_{i,j})_+\le 2 \sum_{j=1}^d\|\bp_j-\br_j\|_1.
\end{eqnarray*}
It is straightforward to show that we need $O(dn^d)$ operations to compute $\cB$ from $\cF$.
\end{proof}
We now give the analog to the error estimate in \cite[Proof of Theorem 1]{AWR17} to the solution of the TOT problem \eqref{TOT}:
\begin{theorem}\label{aproxsolTOT}  Let $\bp_1,\ldots,\bp_d\in \Pi^n_o$ and assume that $\cC\in\otimes^d\R^n$ be given.  Denote by $\tau(\cC,P)$  
the value of TOT problem given by \eqref{TOT}.  Let $\lambda >0$ and denote by $\cA=\exp(-\lambda \cC)$.  Fix $\varepsilon\in(0,1/2)$ and consider  the Sinkhorn algorithm given by  \eqref{defellk}-\eqref{defSA} and the stopping criterion \eqref{stoptalg}.  Let $\cB\in \rU(P)$  be obtained from $\cA_{k^{top}}$ using the Replacement Algorithm given by Lemma \ref{repalg}.
Then 
\begin{eqnarray}\label{upestTOTval}
\langle \cC,\cB\rangle< \tau(\cC,P) +\frac{d\log n}{\lambda}+8d\|\cC\|_\infty\,\varepsilon.
\end{eqnarray}
\end{theorem}
\begin{proof}
Let $\cU^\star =\arg\min\{\langle \cC, \cU\rangle, \cU\in \rU(P)\}$,
$\cF=\cA_{k^{stop}}$ and $\br_j=\s_j(\cF)$ for $j\in[d]$.  Denote by $\cU'\in\rU(R)$ and $\cB\in\rU(P)$ the positive tensors obtained from $\U^\star$ and $\cF$ respectively using the Replacement Algorithm.  Lemma \ref{repalg} and the inequality \eqref{distsAkp} yield
\begin{eqnarray*}
&&\|\cU'-\cU^\star\|_1\le 2\sum_{i=1}^d \|\br_j-\bp_j\|_1\le 4d\varepsilon,\\
&&\|\cB-\cF\|_1\le 2\sum_{i=1}^d \|\bp_j-\br_j\|_1\le 4d\varepsilon.
\end{eqnarray*}
As $\cF$ is a rescaling of $\cA$ Theorem \ref{Ulambform} yields 
\begin{eqnarray*}
\langle \cC,\cF\rangle - \frac{H(\cF)}{\lambda}\le \langle \cC,\cU'\rangle - \frac{H(\cU')}{\lambda}.
\end{eqnarray*}
Hence
\begin{eqnarray*}
\langle \cC,\cF\rangle \le   \langle \cC,\cU'\rangle - \frac{H(\cU')}{\lambda}.+\frac{H(\cF)}{\lambda}< \langle \cC,\cU'\rangle + \frac{d\log n}{\lambda}.
\end{eqnarray*}
Thus
\begin{eqnarray*}
&&\langle \cC,\cB\rangle=\langle \cC,\cB-\cF\rangle +\langle \cC,\cF-\cU'\rangle +\langle \cC,\cU'-\cU^\star\rangle +\langle \cC,\cU^\star\rangle<\\
&&\|\cC\|_\infty\, \|\cB-\cF\|_1 + \frac{d\log n}{\lambda}+\|\cC\|_\infty\, \|\cU'-\cU^\star\|_1+\tau(\cC,\bp_1,\ldots,\bp_d)=\\ 
&&\tau(\cC,P) +\frac{d\log n}{\lambda}+8d\|\cC\|_\infty\,\varepsilon.
\end{eqnarray*}
\end{proof}
\begin{corollary}\label{apprTOTsolest}  Let $\bp_1,\ldots,\bp_d\in \Pi^n_o$ and assume that $\cC\in\otimes^d\R^n$ be given.  Denote by $\omega$ the difference between the maximum and the minimum entry of $\cC$.  Then for a given $\delta>0$ one can find $\cB\in\rU(P)$ in $O(\frac{\omega^3d^4 n^{d+1}\log n}{\delta^3})$ operations such that
\begin{eqnarray*}
\langle C,\cB\rangle\le  \tau(\cC,P)+\delta.
\end{eqnarray*}
\end{corollary}
\begin{proof}  Observe that $\langle \cJ_d,\cU\rangle=1$ for a probability tensor $\cU$.  Let $\mu$ be minimum entry of $\cC$.  Then $\langle \cC,\cU\rangle=\langle (\cU-\mu \cJ_d),\cU\rangle +\mu$ for a probabiity vector $\cU$.  Hence $\tau(\cC,\bp_1,\ldots,\bp_d)=\tau(\cC-\mu\cJ_d,\bp_1,\ldots,\bp_d) +\mu$.  Without loss of generality we can assume that $\cC\ge 0$ and $\cC$ has at least one zero entry.  Thus $\|\cC\|_\infty=\omega$.  Without loss of generality we can assume that $\omega>0$.  In \eqref{upestTOTval} choose
\begin{eqnarray*}
\lambda=\frac{2d\log n}{\delta}, \quad \varepsilon=\min(\frac{1}{4}, \frac{\delta}{16 d\omega}).
\end{eqnarray*}
Then $\langle C,\cB\rangle\le  \tau(\cC,P)+\delta$.  
Let $\cA=\exp(-\lambda \cC)$.  Then the maximum entry of $\cC$ is $1$ and the minimum entry is $\eta=\exp(-\lambda \omega)$.  Hence $\|\cA\|_1<n^d$.  Thus 
\begin{eqnarray*}
\log\frac{\|\cA\|_1}{\eta}\le d\log n+\frac{2d\omega\log n}{\delta}=O(\frac{d\omega\log n}{\delta}).
\end{eqnarray*}
Recall that to find $\cB$ we need to perform $k^{stop}$ iterations of SA. 
One step of SA need $O(dn^d)$ operations.  Use Theorem \ref{stopest} and our choice of $\varepsilon$ to deduce the corollary.  
\end{proof}
\subsection{Sinkhorn algorithm for nonnegative tensors}\label{SAnt}
Let $f\in \rC^1(\R^m)$.  
Denote by $\nabla f(\x)\in\R^m$ the gradient of $f$.  Assume that $\mathbf{U}\subset \R^m$ is a nontrivial subspace of $\R^m$ of dimension $l$.  Denote by $\pi_{\mathbf{U}}:\R^m \to \mathbf{U}$ the orthogonal projection.
Then $\nabla_{\mathbf{U}}f(\x)\defeq \pi_{\mathbf{U}}(\nabla f(\x))$ is the gradient of $f$ in the direction of $\mathbf{U}$.

 Let $\mathbf{V}=\mathbf{B}(P,\0)$, $\mathbf{B}$ and $\mathbf{C}$ be defined by \eqref{defsubsC}. 
 For a nonzero vector $\bp\in\R^n$ denote by $\rL(\bp)=\{\y\in \R^n, \langle \y,\bp\rangle =0\}$.  Thus $\mathbf{V}=\rL(\bp_1)\times\cdots\times\rL(\bp_d)$.  Denote by $\iota _j(\x)=(\0,\ldots,\0,\x,\0,\ldots,\0)$ the isometry  $\iota_j:\R^n\to (\R^n)^d$.
Thus $\mathbf{V}=\oplus_{j=1}^d \iota_j(\rL(\bp_j)$.
Let $\pi_j: (\R^n)^d\to \iota_j(\R^n)$ be the orthogonal projection $(\x_1,\ldots,\x_d)\to(\0,\ldots,\0,\x_j,\0,\ldots,\0)$.   Define $\pi_j(\mathbf{C})=\widetilde{\mathbf{W}}_j$.
Clearly $\widetilde{\mathbf{W}}_j=\iota_j(\mathbf{W_j})$ for a corresponding subspace $\mathbf{W}_j\subset \rL(\bp_j)$.  It is straightforward to show that $\tilde g_{\cA}$ is strictly convex on $\iota_j(\mathbf{W_j})$.  Hence $\iota_j(\mathbf{W_j})\cap \mathbf{B}=\{\0\}$, and $\mathbf{W}_j=\rL(\bp_j)$. 
Let $\pi_{\mathbf{C}}:(\R^n)^d\to \mathbf{C}$ be the orthogonal projection.
Set $\mathbf{V_j}\defeq \pi_{\mathbf{C}}(\iota_j(\rL(\bp_j))$.  It is straightforward to show that $\dim\mathbf{V}_j=n-1$ and $\mathbf{C}=\mathbf{V}_1+\cdots+\mathbf{V}_d$.

We define Sinkhorn algorithm for $\cA$, scalable to $\cU\in\rU(P)$, as follows. 
We retain the steps  \eqref{defAkSA} and \eqref{defSA}.   We replace the steps  \eqref{defellk} and \eqref{stoptalg}
\begin{eqnarray}\label{defellkn}
&&\ell(k-1)=\arg\max\{\|\pi_{\mathbf{V}_j}(\s_j(\cA_{k-1}))\|_1,j\in[d]\},\\
\label{stoptalgn}
&&\max\{\|\pi_{\mathbf{V}_j}(\s_j(\cA_{k-1}))\|_1,j\in[d]\}<\varepsilon.
\end{eqnarray}
Note that if $\mathbf{C}=\mathbf{V}$ then \eqref{defellk} is equal to \eqref{defellkn}, and \eqref{stoptalg} is equal to \eqref{stoptalgn}.
The main result of this subsection is:
\begin{theorem}\label{SAnnten} Let $\bp_1,\ldots,\bp_d\in \Pi^n_o$ and $\varepsilon\in (0,1/2)$.
Assume that $\cA\in \otimes\R^n_+$,  $\s_j(\cA)>0$ for $j\in[n]$, $\cA$ have some zero entries and $\cA$ is scalable to $\cU\in\rU(P)$.  Let $\cA_0=\frac{1}{\|\cA\|_1}\cA$ and $\cA_k$ for $k\ge 1$ be given by Sinkhorn algorithm 
\eqref{defAkSA},  \eqref{defellkn},\eqref{defSA} and the stopping criterion \eqref{stoptalgn}.  Then
\begin{enumerate}[\upshape (a)]
\item Sinkhorn algorithm converges geometrically to a unique $\cU\in\rU(P)$.  
\item The sequence $\cK(\bp_j, \s_j(\cA_k)), k\in\N$ converges geometrically to $0$ for each $j\in[d]$.
\item The sequence $X^{(k)}$ converge geometrically to a minimum point $X^\star$ of $g_\cA$.

\item For $k\ge 1$ and $l=\ell(k)$ the equality \eqref{difxkxk} holds.

\item The following inequality holds: $k^{stop}\le \frac{2(\sqrt{n}+1)^2}{\varepsilon^2}\log \frac{\|\cA\|_1}{\eta}$, where $\eta$ is minimum of positive entries of $\cA$.
\end{enumerate}
\end{theorem}

The proof of the this theorem is similar to the proofs of Theorems \ref{linconvSA} - \ref{stopest}, and we outline the modifications one should do.   
Suppose first that $\mathbf{B}=\{\0\}$.  Then $\tilde g_{\cA}$ is strictly convex on $\mathbf{C}=\mathbf{V}$.   It is straightforward to show that the proofs of Theorems \ref{linconvSA} -  \ref{stopest} apply in this case.

We now discuss the case where $\dim \mathbf{B}>0$.   
Without loss of generality we may assume that $\cA=\cA_0$.  For $X=(\x_1,\ldots,\x_n)\in(\R^n)^d$ let $\cA(X)=[\exp(\sum_{j=1}^d x_{i_j,j})a_{i_1,\ldots,i_d}]$.  
Denote by $\nabla_j \tilde g_{\cA}(X)$,  $\tilde \nabla_{\x_j}\tilde g_{\cA}(X)$,  $\nabla_{\mathbf{B}} \tilde g_{\cA}(X)$ and 
$\nabla_{\mathbf{C}} \tilde g_{\cA}(X)$ the gradient of  $\tilde g_{\cA}$ at $X$ in the direction of the subspace $\mathbf{V}_j$, $\iota_j(\rL(\bp_j))$, $\mathbf{B}$ and $\mathbf{C}$ respectively.  
Recall that $\tilde \nabla \tilde g_{\cA}(X)$ is the gradient of $\tilde g_{\cA}$ at $X$ in the direction of $\mathbf{V}$.  Hence
\begin{eqnarray*}
\tilde \nabla \tilde g_{\cA}(X)=\sum_{j=1}^d \tilde \nabla_{\x_j} \tilde g_{A}(X), \quad
\|\tilde \nabla \tilde g_{\cA}(X))\|^2=\sum_{j=1}^d \|\tilde \nabla_{\x_j} \tilde g_{A}(X)\|^2.
\end{eqnarray*}
From the definition of $\mathbf{B}$ it follows that for $Y=(\y_1,\ldots,\y_n)\in \mathbf{B}$ we have the equality $\exp(\sum_{j=}^n y_{i_j})a_{i_1,\ldots,i_d}=a_{i_1,\ldots,i_d}$.
Hence for $X\in\mathbf{C}$ we have the equality $\cA(X+Y)=\cA(X)$.  In particular $\tilde g_\cA(X+Y)=\tilde g_{\cA}(X)$.  Hence
\begin{eqnarray*}
&&\nabla_{\mathbf{B}} \tilde g_{\cA}(X)=\0,  \tilde \nabla \tilde g_{\cA}(X)= \nabla_{\mathbf{C}} g_{\cA}(X), \quad X\in \R^{n\times d},\\
&&\tilde\nabla \tilde g_{\cA}(X)=\tilde\nabla \tilde g_{\cA}(X+Y)=
\tilde\nabla g_{\cA}(\pi_{\mathbf{C}}(X) \textrm{ for }X\in \mathbf{V}, Y\in\mathbf{B}.
\end{eqnarray*}
We claim that 
\begin{eqnarray}\label{ineqnabxj}\quad\quad
\|\tilde\nabla_{\x_j} \tilde g_{\cA}(X)\|\le \|\nabla_j\tilde g_{\cA}(X)\| \textrm{ for } X\in\mathbf{V}.
\end{eqnarray}
Clearly, it is enough to assume that $\|\tilde\nabla_{\x_j} \tilde g_{\cA}(X)\|>0$.  Let 
\begin{eqnarray*}
\bw(j)=\frac{1}{\|\tilde\nabla_{\x_j}\tilde g_{\cA}(X)\|}\tilde\nabla_{\x_j}\tilde g_{\cA}(X) \in \iota_j(\rL(\bp_j)).
\end{eqnarray*}
Hence $\bw(j)=\iota_j(\bw_j)$ for some $\bw_j\in\rL(\bp_j)$.
Then $\|\tilde \nabla_{\x_j}\tilde g_{\cA}(X)\|=\langle \tilde \nabla \tilde g_{\cA}(X), \bw(j)\rangle$.   Let
\begin{eqnarray*}
&&\bv(j)=\pi_{\mathbf{C}}(\iota(\bw_j))\in \mathbf{V}_j\subset\mathbf{C}, \quad \bu(j)=(I-\pi_{\mathbf{C}})(\iota_j(\bw_j))\in \mathbf{B}, \\
&&\bw(j)=\bu(j)+\bv(j), \quad \|\bu(j)\|^2 +\|\bv(j)\|^2=\|\bw(j)\|^2=1.
\end{eqnarray*} 
As $\langle\tilde\nabla \tilde g_{\cA}(X), \bu(j)\rangle=0$ we deduce that 
\begin{eqnarray*}
&&\|\tilde\nabla_{\x_j} \tilde g_{\cA}(X)\|=\langle \tilde\nabla \tilde g_{\cA}(X)\bw(j)\rangle =\langle \tilde\nabla \tilde g_{\cA}(X)\bv(j)\rangle
=\langle\nabla \pi_{\mathbf{V}_j}\tilde g_{\cA}(X), \bv(j)\rangle=\\
&&\langle \nabla_j \tilde g_{\cA}(X), \bv(j)\rangle\le
\| \nabla_j \tilde g_{\cA}(X)\|\| \bv(j)\|\le \| \nabla_j \tilde g_{\cA}(X)\|.
\end{eqnarray*}
Hence \eqref{ineqnabxj} holds.  Let $f=\tilde g_{\cA}|\mathbf{C}$.  Then $f$ is strictly convex, attains minimum and the inequality \eqref{inasgradf} holds.  Theorem \ref{convalgo}  yields that the PM algorithm converges geometrically to the unique minimum point $\hat X^\star=(\hat \x_1^\star,\ldots,\hat\x_d^\star)\in\mathbf{C}$.

We now relate the PM algorithm to SA as in the proof of Theorem \ref{linconvSA}.
We start PM algorithm with $\hat X^{(0)}=(\hat\x^{(0)}_1,\ldots,\hat\x^{(0)}_d)=0$.  
Let $\hat X^{(k)}=(\hat\x^{(k)}_1,\ldots,\hat\x^{(k)}_d)$ be the $k$-th step in PM algorithm.
Denote $\widehat {\cA}_k=\cA(\hat X^{(k)})$.
Observe that $\|\widehat {\cA}_k\|_1=\tilde g_{\cA}(\hat X^{(k)})$.
We claim that 
\begin{equation}\label{kstepMPMA}
\begin{split}
\hat X^{(k)}&= \pi_{\mathbf{C}}(\iota_l(\y^{(k)}_l))+\hat X^{(k-1)}, \quad l=\ell(k-1)\\
\y^{(k)}_l&=\x_l(\widehat{\cA}_{k-1} )+
(-\cK\bigl(\bp_l,\|\widehat{\cA}_{k-1}\|_1^{-1}\s_l(\widehat{\cA}_{k-1})\bigr) +\log\|\widehat{\cA}_{k-1}\|_1)\1.
\end{split}
\end{equation}
Here $\x_l(\widehat{\cA}_{k-1} )$ is given by \eqref{xilcCform}.  From the proof of Theorem \ref{linconvSA} it follows that the minimum of the restriction of $\tilde g_{\cA}$ to  the affine space $\hat X^{(k-1)}+\iota_l(\rL(\bp_l))$ is achieved at $\iota(\y_l^{(k)})+\hat X^{(k-1)}$.   Observe that $\pi_{\mathbf{C}}(\hat X^{(k-1)}+\iota_l(\rL(\bp_l)))=\hat X^{(k-1)}+\mathbf{V}_j$.
As $\tilde g_\cA(X+Y)=\tilde g_{\cA}(X)$ for $Y\in\mathbf{B}$ it follows that $\tilde g_{\cA}(X)=\tilde g_{\cA}(\pi_{\mathbf{C}}(X))$ for $X\in \mathbf{V}$.  Therefore $\hat X^{(k)}=\pi_{\mathbf{C}}(\iota_l(\y^{(k)}_l)+X^{(k-1)})$ is the minimum point of $\tilde g_{\cA}$ restricted to $\hat X^{(k-1)} +\mathbf{V}_j$.

As the sequence $\hat X^{(k)}$ converges geometrically to $\hat X^\star$ we deduce that $\widehat{\cA}_k $ converges geometrically to $\widehat{\cA}_\star=\cA(X^{\star})$.
As in the proof of Theorem  \ref{linconvSA} we claim that $\cA_k$ given by SA is equal to $\|\widehat{\cA}_k\|^{-1}_1 \widehat{\cA}_k$.  We prove this claim by induction on $k\ge 0$.  For $k=0$  this claim trivially holds.  Assume that this claim holds for $k=m$ and assume that $k=m+1$.  Let $\y_l^{(m+1)}$ be defined as in \eqref{kstepMPMA}.
Then 
\begin{eqnarray*}
\widehat{\cA}_{m+1}=\widehat{\cA}_{m}(\pi_{\mathbf{C}}(\iota_l(\y_l^{(m+1)})))=\widehat{\cA}_{m}(\iota_l(\y_l^{(m+1)})).
\end{eqnarray*}
The $m+1$ step of SA is given by \eqref{defAkSA} and  \eqref{defellk}.  Hence $\cA_{m+1}=\cA_m(\iota_l(\x_l(\cA_m)))$.  The induction assumption $\cA_m=\|\widehat{\cA}_m\|^{-1}_1 \widehat{\cA}_m$ yields that $\cA_{m+1}=t_{m+1} \widehat{\cA}_{m+1}$.  As $\cA_{m+1}$ is a probability tensor we deduce that $t_{m+1}=\|\widehat{\cA}_{m+1}\|^{-1}_1$.  Let $\cA_\star\defeq\|\widehat{\cA}_\star\|^{-1}_1 \widehat{\cA}_\star$.  Thus $\cA_k, k\in\N$ converges geometrically to $\cA_\star$.  We claim that $\cA^\star\in\rU(P)$.

Clearly $\nabla_{\mathbf{C}}\tilde g_{\cA}(\hat X^\star)=\0$.
As $\nabla_{\mathbf{B}}\tilde g_{\cA}(\hat X^\star)=\0$
we deduce that $\tilde\nabla\tilde g_{\cA}(\hat X^\star)=\0$.
Hence $\nabla \tilde g_{\cA}(\hat X^\star)=t^\star P$ for some $t^\star>0$.  Thus $\s_j(\widehat{\cA}_\star)=t^\star\bp_j$ for $j=\in[d]$.  Therefore $\cA\in\rU(P)$.

The proof of parts (b), (c)  and (d) this theorem follows from the arguments of the proofs of the corresponding parts of Theorem \ref{linconvSA}.  To deduce part (e) from the proof of Theorem \ref{stopest} we need to recall the inequality \eqref{ineqnabxj} which is equivalent to
\begin{eqnarray*}
\|\s_j(\cA(\x))-\frac{\langle \s_j(\cA(\x)),\bp_j\rangle}{\|\bp_j\|^2} \bp_j\|\le \|\pi_{\mathbf{V}_j}(s_j(\cA(\x))\|, \quad \x\in\mathbf{V}.
\end{eqnarray*}
 \section*{Acknowledgment}
The author is partially supported by the Simons Collaboration Grant for Mathematicians.  
\bibliographystyle{plain}

 \appendix

\section{Partial minimization of  convex functions}\label{appendix:pm}
Minimization of convex functions is a classical extensively studied subject \cite{Nes83, NN85, Nes15}.  In this Appendix we discuss a very specific minimization algorithm named by us a partial minimization algorithm.  This algorithm enables us to show that the convergence of Sinkhorn algorithm is linear.  We do not claim that this algorithm is good to find a minimum of a smooth convex function.

The following result is well known:
\begin{lemma}\label{existminlem}
Let $f\in\rC^2(\R^m)$ be strictly convex.  Then the following conditions are equivalent:
\begin{enumerate}[\upshape (a)]
\item The function $f$ has a unique minimum $\x^\star\in\R^m$.
\item The following condition hold:
\begin{eqnarray}\label{maxinfcond}
\lim_{\|\x\|\to \infty} f(\x)=\infty.
\end{eqnarray}
\end{enumerate}
\end{lemma}
\begin{proof}  (a)$\Rightarrow$(b).  Let $\rS^{m-1}$ be the $m-1$ dimensional sphere $\|\y-\x^\star\|=1$.    
Fix $\y\in \rS^{m-1}$.  Consider the strict convex function in one variable: $g_{\y}(t)=f(\x^\star+t(\y-\x^\star))$.  Then $g_{\y}'(0)=0$ and $g_{\y}'(1)=\nabla f(\y)^\top (\y-\x^\star)>0$.  Let $\nu=\min\{g_{\y}'(1), \y\in\rS^{m-1}\}$.  Clearly, $\nu>0$.  As $g_{\y}'(t)$ increases for $t>0$ it follows that $g_{\y}'(t)\ge g'_{\y}(1)$ for $t\ge 1$.  In particular, 
\begin{eqnarray*}
g_{\y}(t)\ge g_{\y}(1)+g_{\x}'(1)(t-1)\ge f(\x^\star) +\nu(t-1) \textrm{ for } t\ge 1.
\end{eqnarray*}
Hence $f(\x)\ge f(\x^\star) + \nu(\|\x\|-1)$ if $\|\x-\x^\star\|\ge 1$.  This inequality yields \eqref{maxinfcond}.

\noindent (b)$\Rightarrow$(a)  Fix $\x_0\in\R^m$.  Then there exists $r>0$ such that $\min\{f(\x), \|\x-\x_0\|=r\}>f(\x_0)$.  Let $\min\{f(\x), \|\x-\x_0\|\le r\}=f(\x^\star)$.  Clearly, $\|\x^\star -\x_0\|<r$.  Therefore $\nabla f(\x^\star)=\0$.  As $f(\x)$ is convex we deduce that $f(\x)\ge f(\x_0)$ for each $\x\in\R^n$.  As $f(\x)$ is strictly convex $\x^\star$ is the unique point of minimum of $f$.
\end{proof}
Note that the function $f(x)=e^x, x\in\R$ is strictly convex on $\R$ but $f(x)$ does not have a minimum on $\R$.

In what follows we assume that $f\in\rC^2(\R^m)$ is strictly convex and $\x^\star$ is the unique minimum point of $f$.  Then for each $\x\in\R^m\setminus{\x^\star}$ the  sublevel set
\begin{eqnarray*}
V(t)=\{\y\in\R^n, f(\y)\le t\}, \quad t=f(\x)
\end{eqnarray*}
is a compact strictly convex set, with a $\rC^2$ boundary $\partial V(t)$, and an interior containing $\x^\star$.  Let $t^\star=f(\x^\star)$.  Then $V(t^\star)=\{\x^\star\}$.
Thus $\R^m\setminus \{\x^\star\}$ is parametrized by $\partial V(t), t>t^\star$.

Fix $t_0=f(\x_0)>t^\star$.  Then $f$ is uniformly strictly convex in $V(t_0)$: The eigenvalues of $He(f)(\x), \x\in V(t_0)$ are in a fixed interval $[\alpha(t_0),\beta(t_0)]$ for some $0<\alpha(t_0) \le \beta(t_0)$. 
Thus for each $\x,\y\in V(t_0)$ we have the inequalities:
\begin{eqnarray}\label{strconvond1}
f(\x)+\nabla f(\x)^\top(\y-\x) +\frac{\alpha(t_0)}{2} \|\y-\x\|^2\le f(\y)\le\\ f(\x)+\nabla f(\x)^\top(\y-\x) +\frac{\beta(t_0)}{2} \|\y-\x\|^2.
\label{strconvond2}
\end{eqnarray}
In particular, for $ \x\in V(t_0)$ we have
\begin{eqnarray}\label{lowupbndf}
f(\x^\star)+ \frac{\alpha(t_0)}{2} \|\x-\x^\star\|^2\le f(\x)\le f(\x^\star) +\frac{\beta(t_0)}{2} \|\x-\x^\star\|^2.
\end{eqnarray}
Denote by $\rB(\x, R^2)$ the closed ball $\{\by\in\R^m,\|\bx-\by\|^2\le R^2\}$.  
Let $\kappa(t_0)=\frac{\beta(t_0)}{\alpha(t_0)}$ and define
\begin{eqnarray}\label{defxplusa}
\x^+=\x-\frac{1}{\beta(t_0)}\nabla f(\x), \quad \x^{++}=\x-\frac{1}{\alpha(t_0)}\nabla f(\x).
\end{eqnarray}
In what follows we need the following lemma:
\begin{lemma}\label{xflem}  Assume that $\x\in V(t_0)$.  Let 
\begin{eqnarray}\label{xflem1}
\x^{a}=\x-\frac{2}{\beta(t_0)}\nabla f(\x).
\end{eqnarray}
Then 
\begin{enumerate} [\upshape (a)]
\item 
$f(\x^a)\le f(\x)$.  
\item  $[\x,\x^a]\subset V(t_0)$.
\item  $f(\x)-f(\x^\star)\ge f(\x)-f(\x^+)\ge \frac{\|\nabla f(\x)\|^2}{2\beta(t_0)}$.
\end{enumerate}
\end{lemma}
\begin{proof}  (a) If $\nabla f(\x)=\0$, i.e., $\x=\x^\star$ the (a) trivially holds. Suppose that  $\nabla f(\x)\ne\0$ and assume to the contrary that $f(\x^a)>f(\x)$.  Let $h(t)=f(\x-t\nabla f(\x))$.  Then $h'(0)=-\|\nabla f(\x)\|^2$.  Recall that $h(t)$ is a strict convex function.  Hence there exists $t_1\in (0, \frac{2}{\beta(t_0)})$ such that $h'(t_1)=0$ and $h'(t)>0$ for $t>t_1$.  Thus there exists $t_2\in (t_1, \frac{2}{\beta(t_0)})$ such that $f(\y)=f(\x)$ for $\y=\x-t_2\nabla f(\x)$.  Note that $\y\in V(t_0)$.  This contradicts the inequality \eqref{strconvond2}.

\noindent (b)  As $f(\x^a)\le f(\x)\le t_0$ the convexity of $f$ yields that the interval $[\x,\x^a]$ is in $V(t_0)$.

\noindent(c) Clearly $\x^+=\frac{1}{2}(\x +\x^a)\in [\x,\x^a]$.  Hence 
\begin{eqnarray}  \label{upbetineq}&& f(\x^+)\le \\
&&f(\x)+\nabla f(\x)^\top (\x^+-\x)+\frac{\beta(t_0)}{2}\|\x^+-\x\|^2=f(\x)-\frac{\|\nabla f(\x)\|^2}{2\beta(t_0)}.\notag
\end{eqnarray} 
Therefore (c) holds.
\end{proof}
We now bring the following simple lemma which is basically in \cite{BLS15}:
\begin{lemma}\label{lem1}  Assume that $f\in\rC^2(\R^m)$ is strictly convex and $\x^\star$ is the unique minimum point of $f$.  Fix $\x\in V(t_0)$ and assume that $\x^\star\in \rB(\x, R_0^2)$.   Then we can choose $R_0=R(\x)$ and the following conditions hold:
\begin{eqnarray}\label{eq1}
 R(\x)^2= \frac{2}{\alpha(t_0)}(f(\x)-f(\x^\star))\le\frac{\|\nabla f(\x)\|^2}{\alpha^2(t_0)},\\
\label{eq2}
\x^\star\in \rB(\x^{++}, \frac{\|\nabla f(\x)\|^2}{\alpha^2(t_0)}-\frac{2}{\alpha(t_0)}(f(\x)-f(\x^\star))\subseteq\\
\notag
 \rB(\x^{++}, \frac{\|\nabla f(\x)\|^2}{\alpha^2(t_0)}(1-\frac{1}{\kappa(t_0)})-\frac{2}{\alpha(t_0)}(f(\x^+)-f(\x^\star)),\\
 \label{eq3}
 \frac{\|\nabla f(\x)\|}{\beta(t_0)}\le \|\x-\x^\star\|.
\end{eqnarray}
\end{lemma}
\begin{proof}  As $\x\in V(t_0)$ the left hand side of \eqref{lowupbndf} yields that $\x^\star\in\rB(\x,R(\x)^2)$, where $\R(\x)^2$ is given by \eqref{eq1}.  Clearly
\begin{eqnarray*}
\|\x^\star -\x^{++}\|^2=\|(\x^\star -\x+\frac{1}{\alpha(t_0)}\nabla f(\x)\|^2=\\
\|(\x^\star -\x\|^2 +
\frac{2}{\alpha(t_0)}\nabla f(\x)^\top (\x^\star-\x) +\frac{\|\nabla f(\x)\|^2}{\alpha^2(t_0)}.
\end{eqnarray*}
As $\x^\star,\x\in V(t_0)$ \eqref{strconvond1} yields:
\begin{eqnarray}\label{convcondxs}
f(\x^\star)\ge f(\x)+\nabla f(\x)\trans(\x^\star-\x)+\frac{\alpha(t_0)}{2}\|\x^\star - \x\|^2.
\end{eqnarray}
Thus
\begin{eqnarray*}
\|\x^\star-\x^{++}\|^2\le 
\frac{\|\nabla f(\x)\|^2}{\alpha^2(t_0)} -\frac{2}{\alpha(t_0)}(f(\x)-f(\x^\star)).
\end{eqnarray*}
This proves the first part of  \eqref{eq2}.  Hence the inequality in \eqref{eq1} holds.
Use part (c) of Lemma \ref{xflem}
to replace $f(\x)$ in the first part of \eqref{eq2} by a smaller quantity $f(\x^+)+\frac{\|\nabla f(\x)\|^2}{2\beta(t_0)}$ to obtain the second part of \eqref{eq2}.  

Combine \eqref{upbetineq} with \eqref{lowupbndf} to deduce 
\begin{eqnarray*}
\frac{\|\nabla f(\x)\|^2}{2\beta(t_0)}\le f(\x)-f(\x^+)\le f(\x)-f(\x^\star)\le \frac{\beta(t_0)}{2}\|\x-\x^\star\|^2. 
\end{eqnarray*}
This show the inequality \eqref{eq3}.  
\end{proof}
We now state briefly our algorithm for the norm $\|\cdot\|_s, s\in [1,2]$.
(We are mainly interested in the cases $s=1,2$ but for simplicity of exposition we let $s\in[1,2]$.) 
Let $\mathbf{V}_1,\ldots,\mathbf{V}_d$ be nontrivial proper subspaces of $\R^m$ such that $\mathbf{V}_1+\cdots+\mathbf{V}_d=\R^m$.  Denote by $\pi_j:\R^m\to \mathbf{V}_j$ the orthogonal projection.  Define $\nabla_j f(\x)\defeq \pi_j(\nabla f(\x))$.  We assume the following assumption that holds in our case:
\begin{eqnarray}\label{inasgradf}
\|\nabla f(\x)\|^2\le \sum_{j=1}^d \|\nabla_j f(\x)\|^2, \quad \x\in\R^m.
\end{eqnarray}

Fix $\bu\in\R^m$ and consider $f_{j,\bu}:\V_j\to\R$, where $f_{j,\bu}(\y)=f(\y+\bu)$.   Lemma \ref{existminlem} yields that $f_{j,\bu}$ achieves its minimum at the unique point $\bv_j(\bu)$.
Then the  PM algorithm is:

\noindent
$\quad\quad$ $\;\;$\textbf{PM Algorithm} 
\begin{itemize}
\item[] Choose $\x_0\in\R^m$. 
\item[] for $k:=0,1,2, \ldots$
\item[] $\quad$$j\in\argmax\{\|\nabla_l f(\x_k)\|_s, l\in[d]\}$
\item[] $\quad$$\x_{k+1}=\x_k+\bv_j(\x_k))$.
\item[]end
\end{itemize} 

We now consider a special case of this algorithm when the assumption \eqref{inasgradf} holds.
Divide the vector $\x=(x_1,\ldots,x_m)$ to $d\ge 2$ groups: $\x=(\x_1,\ldots,\x_d)$, where  $\x_i\in \R^{m_i}$ for $i\in[d]$ and $\sum_{i=1}^d m_i=m$.  (Thus $d\in[m]\setminus\{1\}$.)  View $\x$ as $(\x^j,\x_j)$ where $\x^j\in\R^{m-m_j}$ is obtained from $\x$ by deleting the vector coordinate $\x_j$.  Then $\nabla_j f(\x)\in \R^{m_j}$ the vector of derivatives of $f(\x)$ with respect to the coordinates in $\x_j$. 
Clearly  $\|\nabla f(\x)\|_2^2=\sum_{j=1}^d \|\nabla_j f(\x)\|_2^2$.  Thus \eqref{inasgradf} holds.
Minimize $f(\x)$ with respect to the variable $\x_j$ while keeping all other variable fixed:
\begin{eqnarray}\label{partmin}
\min\{f(\x), \x_j\in\R^{m_j}, \x=(\x^j,\x_j)\}=f(\x^j, \x_j(\x^j)).
\end{eqnarray}
Then $\bv_j(\x)=(\x^j,\x_j(\x^j))$.

We now show that in our algorithm the sequences $\x_k,f(\x_k), k\in\N$ converge linearly to $\x^\star,f(\x^\star)$ respectively:
\begin{theorem}\label{convalgo}  
Assume that $f\in\rC^{2}(\R^m)$ is a strict convex function which has a unique minimum point $\x^\star$.  Suppose $\mathbf{V}_1,\ldots,\mathbf{V}_d$ be nontrivial proper subspaces of $\R^m$ such that $\mathbf{V}_1+\cdots+\mathbf{V}_d=\R^m$. 
Assume that the inequality \eqref{inasgradf} holds.  Denote $\ell=\max\{\dim\mathbf{V}_j, j\in[d]\}$.
Let $\x_0\in\R^m$ and  $\x_k,k\in\N$ be given by PM algorithm.   Set $t_k=f(\x_k)$ for $k\in\Z_+$.  Assume that the eigenvalues of $He(f)(\x), \x\in V(t_k)$ are in the minimal interval $[\alpha(t_k),\beta(t_k)]$, where $0<\alpha(t_k) \le \beta(t_k)$. Denote $\kappa(t_k)=\frac{\beta(t_k)}{\alpha(t_k)}$.  
\begin{enumerate}[\upshape (a)]
\item
If $\x_{k-1}\ne \x^\star$ for some $k\in\N$ then $t_{k-1}>t_k$.
\item The sequences $\{t_k\},\{\beta(t_k)\},\{-\alpha(t_k)\},\{\kappa(t_k)\}, k\in\Z_+$ are nonincreasing sequences which converge to $t^\star,\beta(t^\star),-\alpha(t^\star),\kappa(t^\star)$ respectively.
\item
For each $k\in\N$ and $s\in[1,2]$ the following inequalities hold:
\end{enumerate}

\begin{eqnarray}\label{convalgo1}
 &&f(\x_k)-f(\x^\star)\le \\ \notag
  &&(f(\x_0)-f(\x^\star))(1-\frac{1}{d\ell^{(2-s)/s} \kappa(t_0)})\prod_{i=1}^{k-1}(1-\frac{1}{(d-1)\ell^{(2-s)/s} \kappa(t_i)})\le \\  \label{convalgo3}
  &&\frac{\|\nabla f(\x_0)\|^2}{2\alpha(t_0)}(1-\frac{1}{d\ell^{(2-s)/s} \kappa(t_0)})\prod_{i=1}^{k-1}(1-\frac{1}{(d-1)\ell^{2-s}\kappa(t_i)})\le \\ \notag
&&\frac{\|\nabla f(\x_0)\|^2}{2\alpha(t_0)}(1-\frac{1}{d\ell^{2-s}\kappa(t_0)})(1-\frac{1}{(d-1)\ell^{2-s}\kappa(t_0)})^{k-1},\\
\label{convalgo2}
&&\|\x_k-\x^\star\|^2\le \\\notag
&&\frac{\|\nabla f(\x_0)\|^2}{\alpha(t_k)\alpha(t_0)}(1-\frac{1}{d\ell^{2-s}\kappa(t_0)})\prod_{i=1}^{k-1}(1-\frac{1}{(d-1)\ell^{2-s}\kappa(t_i)}).
\end{eqnarray}
\end{theorem}
\begin{proof}  
Recall the following norm inequalities for $s\in [1,2]$ and $\y\in \R^l$:
\begin{eqnarray*}
\|\y\|\le \|\y\|_s, \quad \|\y\|_s\le l^{(2-s)/(2s)}\|\y\|.
\end{eqnarray*}
(The second inequality follows from the H\"older inequality.)
Combine the above first inequality  with
the inequality \eqref{inasgradf} with 
to deduce
\begin{eqnarray*}
\max\{\|\nabla_l f(\x_k)\|_s, l\in[d]\}\ge\max\{\|\nabla_l f(\x_k)\|, l\in[d]\}\ge \frac{\|\nabla f(\x)\|}{\sqrt{d}}.
\end{eqnarray*}
(a) Clearly if $\x_{k-1}=\x^\star$ then $\x_p=\x^\star$ for $p\ge k$.
Assume that $\x_{k-1}\ne \x^\star$.  Then $\|\nabla f(\x_{k-1})\|>0$.  
Let $j_{k-1}\in\argmax \{\|\nabla_l f(\x_{k-1})\|_s, l\in[d]\}$.  Then $\|\nabla_{j_{k-1}} f(\x_{k-1})\|>0$.  Hence $t_{k-1}>t_k$.

\noindent
(c)  First we show the inequality \eqref{convalgo1} for $k=1$.  Assume that $j_0\in\argmax\{ \|\nabla_l f(\x_0)\|_s, l\in[d]\}$.  Hence 
\begin{eqnarray*}
\ell^{(2-s)/(2s)}\|\nabla _{j_0} f(\x_0)\|\ge
\|\nabla _{j_0} f(\x_0)\|_s\ge  \frac{\|\nabla f(\x_0)\|}{\sqrt{d}}. 
\end{eqnarray*}
Let $g(\y)=f(\y+\x_0)=f_{j_0,\x_0}(\y)$ for $\y\in\mathbf{V}_j$.  Thus $g$ is a strictly convex function on $\mathbf{V}_j$, whose Hessian is a submatrix of the Hessian of $f$.
Hence the eigenvalues of the Hessian of $g$ are also in the interval $[\alpha(t_0),\beta(t_0)]$.  Recall that $\argmin g=\bv_{j_0}(\x_0)$.   We now estimate from below $g(\0)-g(\bv_{j_0}(\x_0))$. The lower bound (c) of Lemma  \eqref{xflem} yields:
\begin{eqnarray*}
&&f(\x_0)-f(\x_1)=g(\0)-g(\bv_{j_0}(\x_0))\ge \\ &&\frac{\|\nabla g(\0)\|^2}{2\beta(t_0)}= \frac{\|\nabla f_{j_0}(\x_0)\|^2}{2\beta(t_0)}\ge \frac{\|\nabla f(\x_0)\|^2}{2\beta (t_0)d\ell^{(2-s)/s}}.
\end{eqnarray*}
The inequality \eqref{eq1} yields $f(\x_0)-f(\x^\star)\le \frac{\|\nabla f(x_0)\|^2}{2\alpha(t_0)}$. Assuming that $f(\x_0) >f(\x^\star)$ we obtain
\begin{eqnarray*}
&&\frac{f(\x_1)-f(\x^\star)}{f(\x_0)-f(\x^\star)}=1 - \frac{f(\x_0)-f(\x_1)}{f(\x_0)-f(\x^\star)}\le\\ 
&&1 -  \big(\frac{ \|\nabla f(\x_0)\|^2}{2\beta(t_0) d\ell^{(2-s)/s}}\big)/\big(\frac{\|\nabla f(\x_0)\|^2} {2\alpha(t_0)}\big)=1-\frac{1}{d\ell^{(2-s)/s}\kappa(t_0)}.
\end{eqnarray*}
This proves the first inequality in \eqref{convalgo1} for $k=1$.  

Assume now that $k=2$.  The definition of $\x_1$ yields that $\nabla_{j_0} f(\x_1)=0$.
Hence $\max\{\|\nabla_l f(\x_1)\|_s, l\in[d]\}\ge \frac{\|\nabla f(\x_1)\|}{\sqrt{d-1}}$.
Use the same arguments as above to show that $f(\x_2)-f(\x^\star)\le (f(\x_1)-f(\x^\star))(1-\frac{1}{(d-1)\ell^{(2-s)/s} \kappa(t_1)})$. Hence \eqref{convalgo1} holds for $k=2$.  Similarly, the inequality  \eqref{convalgo1} holds for each $k\ge 2$.

Use the inequality \eqref{eq1} to deduce the inequality in \eqref{convalgo3}.
As $\kappa(t_k)\le \kappa(t_0)$ for each $k\in\N$ we deduce the inequality below \eqref{convalgo3}.
According to Lemma \ref{lem1} $\x^\star\in \rB(\x_k,R^2(\x_k))$.   Use \eqref{convalgo3} to deduce \eqref{convalgo2}.

\noindent
(b) As $\{t_k\},k\in\Z_+$ is a nonincreasing sequence we deduce that $V(t_k)\subseteq V(t_{k-1})$ for $k\in\N$.  Hence the sequence $\{\alpha(t_k)\},k\in\N$ is a nonincreasing, and the sequences $\{\beta(t_k)\},k\in\N$ and $\{\kappa(t_k)\},k\in\N$
are nondecreasing.  The equality $\lim_{k\to\infty} t_k=t^\star$ follows from \eqref{convalgo1}.   The inequality \eqref{convalgo2} yields
\begin{eqnarray*}
\lim_{k\to\infty} \x_k=\x^\star,\lim_{k\to\infty}\alpha(t_k)=\alpha(t^\star),\lim_{k\to\infty}\beta(t_k)=\beta(t^\star),\lim_{k\to\infty}\kappa(t_k)=\kappa(t^\star).
\end{eqnarray*}
\end{proof}
\section{Estimats of certain norms}
\begin{lemma}\label{pqdistest}
Let $\bp,\bq\in\Pi^n$, and $s(\bq,\bp)=\langle \bq,\bp\rangle/\langle\bp,\bp\rangle$.  Then 
the following statement holds:
\begin{enumerate}[\upshape (a)]
\item
\begin{eqnarray*}
&&0\le s(\bq,\bp) \le \frac{(n-1)p_{\max}}{(n-1)p_{\max}^2+(1-p_{\max})^2}, \quad p_{\max}=\max(p_1,\ldots,p_n)\ge 1/n,\\
&&\max\{s(\bq,\bp), \bq,\bp\in\Pi^n\}=\frac{\sqrt{n}+1}{2},\\
&&s(\bq,\frac{1}{n}\1)=1,\\
&&|1-s(\bq,\bp)|=\frac{|\langle \bq-\bp,\bp\rangle|}{\langle\bp,\bp\rangle}\le \frac{\|\bq-\bp\|}{\|\bp\|}\le \sqrt{n}\|\bq-\bp\|_1\le \sqrt{2n\cK(\bp||\bq)}.
\end{eqnarray*}
\item
\begin{eqnarray*}
 &&\|\bq -\frac{\langle \bq,\bp\rangle}{\|\bp\|^2}\bp\|_1\ge |1- \frac{\langle \bq,\bp\rangle}{\|\bp\|^2}|=|1-s(\bq,\bp)|,\\
&&2\|\bq -\frac{\langle \bq,\bp\rangle}{\|\bp\|^2}\bp\|_1\ge \|\bq-\bp\|_1,\\
&& \|\bq-\bp\|_1+|1- \frac{\langle \bq,\bp\rangle}{\|\bp\|^2}|\ge \|\bq -\frac{\langle \bq,\bp\rangle}{\|\bp\|^2}\bp\|_1.
\end{eqnarray*}
\item
\begin{eqnarray*}
\|\bq -\frac{\langle \bq,\bp\rangle}{\|\bp\|^2}\bp\|_1\le (\sqrt{n}+1)\sqrt{2\cK(\bp||\bq)}.
\end{eqnarray*}
\end{enumerate}
\end{lemma}
\begin{proof} (a)  It is straightforward to show
\begin{eqnarray*}
&&\max\{\langle \bq,\bp\rangle, \bq\in\Pi^n\}=p_{\max},\\
&&\arg\min\{\langle \bp, \bp\rangle,  \bp\in\Pi^n,p_{\max}=t,\}=(t, \frac{1-t}{n-1},\ldots,\frac{1-t}{n-1}),\\
&&\min\{\langle \bp, \bp\rangle,  \bp\in\Pi^n,p_{\max}=t,\}=t^2+\frac{(1-t)^2}{n-1}.
\end{eqnarray*}
This shows the first inequality of (a).  Next observe that
\begin{eqnarray*}
&&\arg\max\{ \frac{(n-1)t}{(n-1)t^2+(1-t)^2}, t\in\R\}=\frac{1}{\sqrt{n}},\\
&&\max\{ \frac{(n-1)t}{(n-1)t^2+(1-t)^2}, t\in\R\}=\frac{\sqrt{n}+1}{2}.
\end{eqnarray*}
This shows the second equality of  (a).  The third equality of (a) is straightforward.

We now justify the last statement of (a).  The first equality is straightforward.  The second inequality follows from the Cauchy-Schwarz inequality.  Recall that $\|\x\|\le \|\x\|_1$.  Next observe that $\|\bp\|\ge 1/\sqrt{n}$.  This shows the second inequality.
The last inequality follows from the Pinsker inequality.

\noindent (b)
As $\langle \bq,\bp\rangle\ge 0$ we deduce the first inequality of the lemma:
\begin{eqnarray*}
\|\bq -\frac{\langle \bq,\bp\rangle}{\|\bp\|^2}\bp\|_1\ge \big|\|\bq|\|_1- \|\frac{\langle \bq,\bp\rangle}{\|\bp\|^2}\bp\|_1\big|=|1- \frac{\langle \bq,\bp\rangle}{\|\bp\|^2}|.
\end{eqnarray*}
Observe next
\begin{eqnarray*}
&&\|\bq -\frac{\langle \bq,\bp\rangle}{\|\bp\|^2}\bp\|_1=\|\bq-\bp +(1-\frac{\langle \bq,\bp\rangle}{\|\bp\|^2})\bp\|_1\ge \\
&&\|\bq-\bp\|_1-\|(1-\frac{\langle \bq,\bp\rangle}{\|\bp\|^2})\bp\|_1=\|\bq-\bp\|_1 -|1- \frac{\langle \bq,\bp\rangle}{\|\bp\|^2}|.
\end{eqnarray*}
Thus implies the second inequality of the lemma. Use the triangle inequality for the above expression of $\|\bq -\frac{\langle \bq,\bp\rangle}{\|\bp\|^2}\bp\|_1$ to deduce the third inequality of the lemma.

\noindent (c) Combine the last inequalities of (b) and (a).
\end{proof}

We now show that the  inequality of (c) is sharp asymptotically up to at most a multiplicative factor of $1/\sqrt{\log a_n}$. where $a_n, n\in\N$ is a  slowly increasing to infinity, with $a_1\ge 2$:
\begin{eqnarray*}
\bp=(\frac{1}{\sqrt{n}},\frac{1-1/\sqrt{n}}{n-1},\ldots,\frac{1-1/\sqrt{n}}{n-1}),\bq=(1-\frac{1}{a_n}, \frac{1}{(n-1)a_n},\ldots,\frac{1}{(n-1)a_n}).
\end{eqnarray*}
Then
\begin{eqnarray*}
&\|\bq -\frac{\langle \bq,\bp\rangle}{\|\bp\|^2}\bp\|_1=\|\bq-\bp +(1-\frac{\langle \bq,\bp\rangle}{\|\bp\|^2})\bp\|_1\ge\\
&|1- \frac{\langle \bq,\bp\rangle}{\|\bp\|^2}| -\|\bq-\bp\|_1\ge \frac{\langle \bq,\bp\rangle}{\|\bp\|^2} -3\ge\frac{n}{4}\frac{1-1/a_n}{\sqrt{n}}-3=\frac{\sqrt{n}(1-1/a_n)}{4} -3,\\
&\cK(\bp||\bq)=-\frac{1}{\sqrt{n}}\log\bigl(\sqrt{n}(1-1/a_n)\bigr) +\bigl(1-\frac{1}{\sqrt{n}}\bigr)\log \bigl(a_n(1-1/\sqrt{n})\bigr).
\end{eqnarray*}

\end{document}